\newif\ifarxiv
\arxivtrue
\pdfoutput=1

\ifarxiv
    \documentclass[twoside,11pt]{article}
  \usepackage{cancel}
\else
    \documentclass[pmlr]{jmlr}
\fi

\ifarxiv
    \usepackage[usenames,dvipsnames,svgnames,table]{xcolor}
    \usepackage{tabularx}
\fi

\usepackage{enumitem}
\usepackage[disable,textsize=tiny]{todonotes} 
\usepackage{natbib}
\usepackage{graphicx}
\usepackage{fancyhdr}
\ifarxiv
    \usepackage[hyperindex,
                linktocpage=true,
                colorlinks=true,
                linkcolor=blue,
                urlcolor=blue,
                citecolor=blue,
                anchorcolor=blue
                ]{hyperref}  
    \usepackage{amsmath,amssymb}
    \usepackage{amsthm}
\fi
\usepackage[capitalize]{cleveref} 

\ifarxiv
    \theoremstyle{plain}
    \newtheorem{theorem}{Theorem}
    \newtheorem{lemma}[theorem]{Lemma}
    \newtheorem{corollary}[theorem]{Corollary}
    
    \newtheorem{remark}[theorem]{Remark}
\fi

\definecolor{darkred}{rgb}{.5,0,0}
\definecolor{darkgreen}{rgb}{0,.5,0}
\definecolor{darkblue}{rgb}{0,0,.5}
\definecolor{darkorange}{rgb}{.8,.4,0}
\newcommand{\KL}[2]{\operatorname{RE}(#1\|#2)}

\newcommand{\half}{\tfrac{1}{2}}

\newcommand{\Reals}{{\mathbb R}}

\newcommand{\indicator}[1]{[\![#1]\!]} 

\newcommand{\D}{\mathrm{d}} 

\setlist{nosep}

\newcommand\numberthis{\addtocounter{equation}{1}\tag{\theequation}}

\newcommand{\SB}{Soft-Bayes}

\newcommand{\loss}{\ell}
\newcommand{\losst}{\loss_t}

\newcommand{\bigloss}{C} 

\newcommand{\cR}{\mathcal R}
\newcommand{\cD}{\mathcal D}
\newcommand{\cW}{\mathcal W}
\newcommand{\mixt}{M}
\newcommand{\mixtt}{\mixt_t}
\newcommand{\mixttpo}{\mixt_{t+1}}
\newcommand{\mixttoT}{\mixt_{1:T}}
\newcommand{\mixttot}{\mixt_{1:t}}
\newcommand{\mixtpret}{\mixt_{<t}}
\newcommand{\mdlset}{\mathcal{M}}
\newcommand{\nmdl}{N} 
\newcommand{\mdl}{p}
\newcommand{\mdli}{\mdl^i}
\newcommand{\mdlit}{\mdl^i_t}

\newcommand{\mdljt}{\mdl^j_t}
\newcommand{\mdlitoT}{\mdl^i_{1:T}}
\newcommand{\mdlitpo}{\mdl^i_{t+1}}
\newcommand{\spe}{\tilde{p}}
\newcommand{\spei}{\spe^i}
\newcommand{\speit}{\spe^i_t}
\newcommand{\speitoT}{\spe^i_{1:T}}

\newcommand{\wmdl}{w}
\newcommand{\wmdltpo}{\wmdl_{t+1}}
\newcommand{\wmdli}{\wmdl^i}
\newcommand{\wmdlit}{\wmdl^i_t}
\newcommand{\wmdljt}{\wmdl^j_t}
\newcommand{\wmdlitpo}{\wmdl^i_{t+1}} 
\newcommand{\wmdliTpo}{\wmdl^i_{T+1}}
\newcommand{\wmdliTstart}{\wmdli_{\Tstarti}}

\newcommand{\fix}{A}
\newcommand{\fixt}{\fix_t}
\newcommand{\fixtoT}{\fix_{1:T}}
\newcommand{\wfix}{a}
\newcommand{\wfixi}{a^i}

\newcommand{\obs}{x}
\newcommand{\allObs}{\mathcal{X}}
\newcommand{\eps}{\varepsilon}

\newcommand{\mdldeltai}{\mdl^{\delta i}}

\newcommand{\rate}{\eta} 
\newcommand{\ratebar}{\bar{\rate}}
\newcommand{\rateo}{\rate_1} 
\newcommand{\ratez}{\rate_0} 
\newcommand{\ratet}{\rate_t}

\newcommand{\ratebart}{\ratebar_t}
\newcommand{\ratetpo}{\rate_{t+1}}
\newcommand{\ratetmo}{\rate_{t-1}}
\newcommand{\rateT}{\rate_{T}}
\newcommand{\rateTpo}{\rate_{T+1}}
\newcommand{\rateTstarti}{\rate_{\Tstarti}}
\newcommand{\rateTstartimo}{\rate_{\Tstarti-1}}

\newcommand{\ratei}{\rate^i}
\newcommand{\rateit}{\rate^i_t}
\newcommand{\rateitpo}{\rate^i_{t+1}}
\newcommand{\rateiTpo}{\rate^i_{T+1}}
\newcommand{\ratebari}{\ratebar^i}
\newcommand{\ratebarit}{\ratebar^i_t}
\newcommand{\ratebaritpo}{\ratebar^i_{t+1}}
\newcommand{\ratebariTpo}{\ratebar^i_{T+1}}

\newcommand{\lnrate}{\ln_{\rate}}
\newcommand{\mdlfrac}{q}
\newcommand{\mdlfracmax}{Q}

\newcommand{\bestset}{\mdlset^*} 
\newcommand{\bestsett}{\bestset_t}
\newcommand{\bestsetTpo}{\bestset_{T+1}}
\newcommand{\bestsetTstarti}{\bestset_{\Tstarti}}
\newcommand{\nbestset}{m}
\newcommand{\nbestsett}{\nbestset_t} 
\newcommand{\nbestsettpo}{\nbestset_{t+1}}
\newcommand{\nbestsetTpo}{\nbestset_{T+1}}
\newcommand{\nbestsetTstarti}{\nbestset_{\Tstarti}}
\newcommand{\nbestsetTstartimo}{\nbestset_{\Tstarti-1}}

\newcommand{\nshift}{K}
\newcommand{\idxshift}{k}
\newcommand{\fixidx}[1]{\fix^{#1}}
\newcommand{\fixk}{\fixidx{\idxshift}}
\newcommand{\wfixik}{\wfixi_{\idxshift}}
\newcommand{\Tshiftidx}[1]{T_{#1}}
\newcommand{\Tshiftidxend}[1]{T_{#1'}}
\newcommand{\Tshiftk}{\Tshiftidx{\idxshift}}
\newcommand{\Tshiftkend}{\Tshiftidxend{\idxshift}}
\newcommand{\Tshiftkpo}{\Tshiftidx{\idxshift+1}}

\newcommand{\Tstarti}{T_i} 

\begin{document}


\ifarxiv
    \title{Soft-Bayes: Prod for Mixtures of Experts with Log-Loss}
    \author{Laurent Orseau$^1$ \and Tor Lattimore$^1$ \and Shane Legg$^1$}
    \date{
        \{lorseau,lattimore,legg\}@google.com \\%
        DeepMind, London, UK
    }
\fi

\maketitle

\begin{abstract}
We consider prediction with expert advice under the log-loss with the goal of deriving efficient and robust algorithms. We argue that existing algorithms such as exponentiated gradient, online gradient descent and online Newton step do not adequately satisfy both requirements.
Our main contribution is an analysis of the Prod algorithm that is robust to any data sequence and runs in linear time relative to the number of experts in each round.
Despite the unbounded nature of the log-loss, we derive a bound that is independent of the largest loss and of the largest gradient, and depends only
on the number of experts and the time horizon.
Furthermore we give a Bayesian interpretation of Prod and adapt the algorithm to derive a tracking regret.
\end{abstract}

\section{Introduction}\label{sec:intro}

Sequence prediction is a simple problem at the core of many machine learning problems: Given a sequence of past observations, what is the probability of the next one?
We approach this problem using the prediction with expert advice framework with logarithmic loss.
The log-loss is arguably the most fundamental choice 
because of its connections to probability (minimising the log-loss is maximising the likelihood of the observed data), sequential betting, information
theory and compression (minimising the log-loss corresponds to maximising the compression ratio when compressing using arithmetic coding).

\paragraph{Setup} Let $\allObs$ be a countable alphabet and $\cD$ be the set of probability distributions over $\allObs$.
We consider a game over $T$ rounds, where in each
round $t$ the learner first receives the predictions of $N$ experts $(\mdlit)_{i=1}^\nmdl$ with $\mdlit \in \cD$.
The learner then chooses their own prediction $\mixtt \in \cD$ and the next symbol $\obs_t \in \allObs$ is revealed.
The learner then suffers an \emph{instantaneous loss} at round $t$ of
\begin{align*}
\losst(\mixtt) = -\log\left(\mixtt(\obs_t)\right)
\end{align*}
and the cycles continue up to round $T$, which may or may not be known in advance.
The learner would like to combine the advice of the experts so as to make the loss as small as possible, which corresponds to good prediction/compression. We make no assumptions on the data source; in particular
$\obs_{1:T} = \obs_1,\obs_2,\ldots,\obs_T$ need not be independent and identically distributed, and may even be generated by an adversary. The standard approach in this setting
is to analyse the regret of the predictor $\mixt$ relative to some interesting class of competitors. 
Here we focus on predicting well relative to a fixed convex combination of experts. Let $\cW$ be the $(\nmdl-1)$-dimensional probability simplex:
$\forall \wmdl \in \cW: \sum_{i=1}^\nmdl \wmdli = 1$ and $\forall i\in[\nmdl]: \wmdli \in[0,1]$, where $\wmdli$ is the $i$th component of the vector $\wmdl$.
Define the \emph{regret} (also called redundancy for the log-loss) relative to $a \in \cW$ by
\begin{align}
\cR_T(\wfix) = \sum_{t=1}^T \left[\losst(\mixtt) - \losst\left(\sum_{i=1}^\nmdl \wfixi \mdlit\right)\right]\,.
\label{eq:regret}
\end{align}
This definition of the regret is more demanding than the usual notion of competing with the best single expert in hindsight, which corresponds to competing with `Dirac' experts $\wfix \in \{e_1,e_2,\ldots,e_\nmdl\}$ with $e_i$
the standard basis vectors. In particular, defining the regret as in \cref{eq:regret} forces the learner to exploit the `wisdom of the crowd' by combining the experts predictions, rather than focusing on the 
single best expert in hindsight. This is often crucial because it often happens that no single expert predicts well over all time periods, and the nature
of the log-loss means that even a single round of poor prediction can lead to a large instantaneous loss.

This framework has attracted significant attention over the last three decades, mainly due to its applications to compression and portfolio optimisation; see for example \citet{KV02} and references within.
Our objective is to design algorithms that are (a) linear-time in the number of experts $\nmdl$, (b) robust to the existence of incompetent experts and (c) recover the tracking guarantees of 
fixed share~\citep{herbster1998tracking}.
A variety of approaches are now known for this problem, but none satisfy all of (a), (b) and (c) above.
The minimax optimal and Bayesian solutions are known to achieve logarithmic regret \citep{Cov91}, but there are currently no linear-time implementations and it seems unlikely that one exists.
The main computational challenge for the Bayesian approach is evaluating the normalisation integral. This issue was addressed by \cite{KV02} via a polynomial-time sampling approach, but
unfortunately the solution far from linear in $\nmdl$ and is not suitable for practical implementations when $\nmdl$ is large.
More recently \cite{HAK07} proposed the online Newton step, a pseudo second-order algorithm that also achieves logarithmic regret, but depends on computing a generalised projection for 
which the best-known running time is $O(\nmdl^2)$ per step~\citep{LACL16}. Even in the idealised case that the projection could be computed in $O(1)$, the algorithm depends on maintaining and updating
a covariance matrix and so $O(\nmdl^2)$ running time is unavoidable without some form of dimensionality reduction that weakens the regret guarantees (see \citealp{LACL16}).

All of the algorithms mentioned until now enjoy optimal logarithmic regret.
The exponentiated gradient (EG) algorithm by \cite{HSSW98} has a 
regret bound that looks like
\begin{align}
\cR_T(a) \leq C \sqrt{\frac{T}{2} \log N}\,,
\label{eq:EGbound}
\end{align}
where $C = \max_{t,i} \frac{\mdlit(\obs_t)}{\mixtt(\obs_t)}\leq \max_t \frac{1}{\mixtt(\obs_t)}$. The EG algorithm runs in $O(N)$ time per round, but unfortunately its regret depends on $C$, which may be so large
that it can make \cref{eq:EGbound} vacuous (Section~\ref{sec:largelosses}).
Even worse, the dependence on $C$ is not an artifact of the analysis, but rather a failing of the EG algorithm, which becomes unstable when experts transition from predicting badly to predicting well.
Another algorithm with near-linear running time is online gradient descent (OGD) by \cite{Zin03} (and applied to this setting by \cite{VSH12}), which 
runs in $O(N \log(N))$ time using the fast simplex projection by \cite{DSSC08}.
The regret of this algorithm also depends on the size of the maximum gradient of the loss, however, which leads to bound of the same order as \cref{eq:EGbound}. Note that EG is equivalent to using mirror descent
with neg-entropy regularisation, which makes the projection to the simplex nothing more than normalisation. We will return briefly to alternative regularisation choices for mirror descent in the discussion.

We revisit the Prod algorithm by \cite{cesabianchi2007prod}, which runs in $O(N)$ time per step and was originally designed for obtaining second-order bounds when competing with the best expert rather than the mixture.
The stability of Prod has not gone unnoticed, with recent work by \cite{gaillard2014mlprod} and \cite{SNL14} also exploiting its advantages over exponential weighting.
Since the log-loss is unbounded, it is not immediately suitable for use in Prod, but conveniently the linearised loss \textit{is} semi-bounded. In this sense our algorithm is to Prod what exponentiated gradient is to
exponential weighting.

\paragraph{Contributions}
Our main contribution is an analysis of Prod when competing against a mixture in the log-loss setting (Sections~\ref{sec:offline} and  \ref{sec:self-conf}). 
By tuning the learning rate we are able to show two regret bounds: \\
\begin{minipage}{7cm}
\begin{align}
\cR_T &= O\left(\sqrt{NT \log N}\right) \label{eq:SBboundN}
\end{align}
\end{minipage}
\hspace{1cm}
\begin{minipage}{7cm}
\begin{align}
  \cR_T &= O\left(\sqrt{TC\log N}\right)\,, \label{eq:SBboundeps}
\end{align}
\end{minipage} \\
where $C$ is defined as above.
The first bound (Section~\ref{sec:offline}) eliminates \textit{all dependence} on the arbitrarily large $C$ at the price of a square-root dependence on the dimension $N$. The second bound (Section~\ref{sec:self-conf}) retains a dependence on $C$, but moves 
it inside the square root relative to the EG algorithm. We also prove self-confident bounds (Section~\ref{sec:self-conf}) and analyse a truly online version of Prod that does not need prior knowledge of
the horizon and simultaneously achieves a tracking guarantee (Section~\ref{sec:online}).
To complement the upper bounds we prove lower bounds for EG and OGD showing that in the worst case they suffer nearly \textit{linear} regret (Section~\ref{sec:largelosses}).
Moreover, we give two Bayesian interpretations of Prod as a `slowed down' Bayesian predictor or a mixture of `partially sleeping' experts (Section~\ref{sec:sbayes}).

\paragraph{Notation}
For a natural number $n$ let $[n] = \{1,2,\ldots,n\}$ and $e_i$
be the standard basis vectors (the dimension will always be clear from context).
The number of experts is denoted by $N \geq 1$ and the time horizon is $T \geq 1$.
Let $\cD$ be the set of probability distributions on the countable alphabet $\allObs$, and $\cW$ be the $(N-1)$-dimensional probability simplex
so that for $a \in \cW$ we have $a^i \geq 0$ for all $i$ and $\sum_{i=1}^N a^i = 1$. For $a,w \in \cW$ we define $\KL{a}{b} = \sum_{i:a^i > 0}^N a^i \ln(a^i/b^i)$ to be the relative entropy between
$a$ and $b$.
All of the following analysis only relies on $p^i_t$ through $p^i_t(x_t)$, which is the probability that expert $i$ assigned to the actual observation $x_t$ in round $t$.
For this reason we abbreviate $p^i_t = p^i_t(x_t)$. We also use $p^i_{1:t} = \prod_{s=1}^t p^i_s$. Similarly we abbreviate the prediction $M_t = M_t(x_t)$ and $M_{1:t} = \prod_{s=1}^t M_s$.
We usually reserve $w,w_t \in \cW$ to denote the weights over experts used by an algorithm and $a \in \cW$ for a fixed competitor.
Given $a \in \cW$, we let $A_t = \sum_{i=1}^n a^i p^i_t$ be the prediction of the mixture of experts over $a$ and $A_{1:t} = \prod_{s=1}^t A_s$.
The indicator function is denoted by $\indicator{test}\in\{0,1\}$ and equals 1 if the boolean $test$ is true.

\section{The Soft-Bayes algorithm}\label{sec:sbayes}
As discussed in the introduction, the Prod algorithm was originally designed for prediction with expert advice when competing with the
best expert in hindsight. Let $\wmdl_1 \in \cW$ be the initial (prior) weights,
which we always take to be uniform $\wmdl_1 = \frac{1}{\nmdl}$ unless otherwise stated.
Then in each round $t$ the Prod algorithm predicts using
a mixture over the experts (\ref{eq:mixture}) and updates its weights using a multiplicative update rule (\ref{eq:prod-update}):\newline
\begin{minipage}{6cm}
\begin{align}
M_t = \sum_{i=1}^N \wmdlit p^i_t\,.
\label{eq:mixture}
\end{align}
\end{minipage}
\hspace{2cm}
\begin{minipage}{7cm}
\begin{align}
\wmdlitpo = \frac{\wmdlit(1 - \ratebar \ell^i_t)}{\sum_{j=1}^\nmdl \wmdljt(1 - \ratebar \ell^j_t)}\,,\qquad\qquad
\label{eq:prod-update}
\end{align}
\end{minipage} \\
where $\ell^i_t = \ell_t(e_i)$ is the loss suffered by expert $i$ in round $t$ and $\ratebar \in (0,1)$ is the learning rate.
The algorithm only makes sense if $\ell^i_t \leq 1$, which is usually assumed.
Recall our loss function is $\ell_t:\cW \to \Reals$ is given by $\ell_t(w) = -\log M_t = -\log \sum_{i=1}^N \wmdlit p^i_t$, which is convex but arbitrarily large.
The key idea is to predict using \cref{eq:mixture}, but replace the loss with the linearised loss, $\nabla \ell_t(w_t)^\top w$.
A simple calculation shows that
$\nabla \ell_t(w_t)_i = -p_t^i / M_t \leq 0$.
Therefore the linearised losses are semi-bounded.
If we instantiate Prod, but replace the loss of each expert with the linearised loss, then the resulting algorithm predicts like \cref{eq:mixture} and updates
its weights by
\begin{align}
w_{t+1}^i
= \frac{\wmdlit\left(1 + \ratebar \frac{p_t^i}{M_t}\right)}{\sum_{j=1}^N w^j_t\left(1 + \ratebar \frac{p_t^j}{M_t}\right)}
= \frac{\wmdlit\left(1 + \ratebar \frac{p_t^i}{M_t}\right)}{1 + \ratebar}
= \wmdlit\left(1 - \rate + \rate\frac{p_t^i}{M_t}\right)\,,
\label{eq:update}
\end{align}
where $\rate = \ratebar / (1 + \ratebar)$ so that $\ratebar = \rate / (1 - \rate)$ and the second equality follows from the definition of $M_t = \sum_j w^j_t p^j_t$ and the fact that $\sum_i \wmdlit = 1$.
Notice that the computations of the prediction \cref{eq:mixture} and weight update \cref{eq:update} are both linear in the number of experts $N$.
For reference, the EG and OGD algorithms also predicts like \cref{eq:mixture}, but update their weights by
\begin{align}\label{eq:EGupdate}
\text{EG:} \quad
w_{t+1}^i = \frac{\wmdlit \exp\left(\rate\frac{p^i_t}{M_t}\right)}{\sum_{j=1}^N w_t^j \exp\left(\rate \frac{p^j_t}{M_t}\right)} \qquad\qquad
\text{OGD:} \quad
w_{t+1}^i = \Pi\left(w_t + \rate\frac{ p_t}{M_t}\right)_i \,,
\end{align}
where $\Pi$ is the projection onto the simplex with respect to the Euclidean norm.
A careful examination of the EG update leads to a worrying observation: If $\wmdlit$ is close to zero and $p^i_t$ is close to one, then $p^i_t / M_t$ can be extremely large, causing $\wmdlitpo$ to be close to one and $w^j_{t+1}$ to drop to nearly zero for all $j \neq i$.
This makes EG unstable when the gradients are large. The OGD update can be even worse because the projection has the potential to concentrate the weights on a Dirac after which the regret
can be infinite!
In contrast, Prod behaves more conservatively since%
\footnote{The second equality suggests that Prod and \SB{} are closely related to exponential smoothing for probability estimation~\citep{mattern2016phd}.}
\begin{align*}
\wmdlitpo = \wmdlit\left(1 - \rate + \rate \frac{p^i_t}{M_t}\right)
= (1 - \rate)\wmdlit + \rate \frac{\wmdlit p^i_t}{M_t} \leq (1 - \rate) \wmdlit + \rate\,,
\end{align*}
where the inequality follows from the dominance property that $M_t \geq \wmdlit p^i_t$ for all $i$ and $t$.
This means that even in the most extreme scenarios, the weight increases by at most $\rate$, which is usually tuned to be approximately $T^{-1/2}$.

\paragraph{Bayesian interpretation}
We now give two Bayesian interpretations of this algorithm. The first is to note that if $\rate = 1$, then
\begin{align*}
\wmdlitpo = \frac{\wmdlit \mdlit}{\mixtt} \qquad \text{ and } \qquad  \mixttoT = \sum_{i=1}^\nmdl \wmdli_1 \mdlitoT\,.
\end{align*}
In this case $\wmdlit$ is the posterior of the Bayesian mixture over sources $(p^i)_i$ with prior $w_1$.
While its regret relative to a single expert is at most $\cR_T(e_i) \leq \log N$,
the algorithm does not compete with convex combinations of experts. From a Bayesian perspective, there is no reason to believe that it should because the convex combinations lies outside the class of the learner.
On the other hand, if the learning $\rate$ is chosen to be close to zero, then the Prod update has the effect of `slowing down Bayes' to ensure it does not concentrate too fast on a single promising expert.
The multiplicative/additive nature of the update also means that
experts can make big mistakes while losing at most $(1-\rate)$ of their current weight. In contrast, a `slow' update derived from the exponential weights algorithm that
looks like $\wmdlitpo \propto \wmdlit \exp(\eta \log(\mdlit / \mixtt))$ still reduces the weight of an expert to zero if $\mdlit = 0$.

The second interpretation comes from the
sleeping expert framework~\citep{FSSW97} that allows experts to `fall asleep' and abstain from predicting in some rounds.
If the weights are normalised appropriately, then this is equivalent to assuming the sleeping experts defer their vote to the wakeful, which
for them is equivalent to predicting like the mixture $\mixtt$~\citep{chernov2009prediction}.
We consider a smooth version of this idea, where an expert can be `sleepy'
and predict partially like the mixture.
From a given expert $\mdli$, we build the meta-expert $\spei$ such that for all time steps t:
\begin{align*}
\speit &:= (1-\rate)\mixtt + \rate \mdlit \qquad \text{ and thus }\qquad \speitoT = \prod_{t=1}^T ((1-\rate)\mixtt + \rate\mdlit)\,,
\end{align*}
where $\rate\in(0, 1]$ and $M_t$ is now defined as a mixture of the meta-experts $\spei$:
\begin{align}
\mixttoT &= \sum_{i=1}^\nmdl \wmdli_1 \speitoT \label{eq:sbpred}
\end{align}
Note that since both $\mdl$ and $\mixt$ are predictors
the convex combination of these is also a predictor.
Such self-referential constructions have been noted in the past, for example by \cite{koolen2012putting}.
The main point is that the meta-experts are not normal predictors because they depend on the learner $\mixt$. Nevertheless, they are useful for analysis and intuition.
With this view of $\mixt$ we note that all the usual properties of Bayesian predictors hold. In particular:
\begin{itemize}
\item The posterior weight of the $i$th meta-expert is
\begin{align*}
\wmdlitpo = \wmdli_1 \frac{\spei_{1:t}}{\mixt_{1:t}} = \wmdli_{t} \frac{\spei_t}{\mixtt} = \wmdli_{t}\left(1 - \rate + \rate \frac{\mdlit}{\mixtt}\right)\,.
\end{align*}
\item The Bayes prediction over the class of meta-experts is $\displaystyle \mixtt = \sum_{i=1}^\nmdl \wmdlit \spei_t$\,.
\item The weights are properly normalised: $\displaystyle \sum_{i=1}^\nmdl \wmdlit = 1$ for all $t \in [T]$.
\end{itemize}
Based on these observations we call the algorithm defined by the prediction in \cref{eq:mixture} and updates in \cref{eq:update}, or equivalently by \cref{eq:sbpred}, the \textit{Soft-Bayes} algorithm.

\paragraph{Regret relative to a single expert}
We start the theoretical results with a simple bound on the regret relative to a single expert.

\begin{theorem}
For the \SB{} algorithm,
$\cR_T(e_i) =\ln \frac{\mdlitoT}{\mixttoT} \leq \frac{1}{\rate}\ln \frac{1}{\wmdli_1}$ for all $i$.
\end{theorem}
\begin{proof}
Using dominance and the definition of concavity applied to the function $\ln(\cdot)$:
\begin{align*}
\forall i\in[\nmdl]:&\ln\mixttoT \geq \ln\wmdli_1 \speitoT = \ln \wmdli_1 + \sum_{t=1}^T \ln ((1-\rate)\mixtt+\rate\mdlit) \\
&\geq \ln \wmdli_1 + (1-\rate)\sum_{t=1}^T \ln \mixtt + \rate\sum_{t=1}^T \ln \mdlit
=  \ln \wmdli_1 + (1-\rate)\ln \mixttoT + \rate\ln \mdlitoT\,.
\end{align*}
Therefore $\rate \ln\mixttoT \geq \rate\ln \mdlitoT + \ln \wmdli_1$, and the proof is completed by rearrangement.
\end{proof}

If the goal is to compete with the best expert only, then setting $\rate=1$ is optimal, which makes Prod equivalent to the standard Bayesian algorithm over $(e_i)$.
As an aside, in \cref{sec:disjoint} we present a simple setup
where we recover several well-known algorithms by using Soft-Bayes with specific simple learning rates.

\section{Failure of EG and OGD}\label{sec:largelosses}
As remarked in the introduction, the EG and OGD algorithms can become unstable when the gradients are uncontrolled.
Here we demonstrate this with a carefully crafted example that best illustrates the issue.

\begin{theorem}\label{thm:eg-lower}
If $N = 2$ and $w_1 = (1/2, 1/2) \in \cW$ is the uniform prior, then for any learning rate $\eta > 0$ there exists a sequence of predictors such that
the regret of EG and OGD is $\Omega(T^{1- \varepsilon})$ for all $\varepsilon \in (0,1)$.
\end{theorem}

The proof is given in Appendix~\ref{app:thm:eg-lower} and depends on a simple example where the experts predict Dirac measures with disjoint support (they always disagree).
In the first $T/2$ rounds the first expert is always wrong and for the next $T/2$ rounds the correctness of the experts alternates. If the learning rate is sufficiently large,
then the weights of the EG algorithm oscillate wildly, which leads to a \textit{super-exponential} regret.
The only way to avoid this calamity is to choose a learning rate so small that EG barely learns at all, in which case it suffers near-linear regret.
For OGD the regret can even be infinite in this example.
A naive attempt to fix the EG algorithm is to replace the experts with `meta-experts'
$\mdldeltai$ defined by $\mdldeltai_t := (1-\delta)\mdlit + \delta/|\allObs|$.
This ensures that $\mixtt = \sum_i \wmdlit \mdldeltai_t \geq \delta/|\allObs|\approx 1/c$ for all $t$.
While this does prevent super-exponential regret, it does not solve the problem.
Compared to a mixture of the base experts $\mdli$, the mixture of the meta-experts $\mdldeltai$ can suffer
a regret of $T\delta$.
Hence, considering the bound in \cref{eq:EGbound}, the optimal balance is for $\delta \approx T^{-1/4}$ leading to a bound of $O(T^{3/4})$, which is much worse than the $O(T^{1/2})$ regret
that we prove for Prod. A similar correction is possible for OGD, but does not seem worthwhile in light of the above discussion.

\section{Regret against a mixture}\label{sec:offline}

Recall that for any $a \in \cW$, the mixture predictor is $A_t = \sum_{i=1}^N a^i p^i_t$ and $A_{1:T} = \prod_{t=1}^T A_t$.
Let $\bestset:=\{i\mid\exists t\in [T] :\mdlit = \max_j \mdljt\}$ be
the set of experts that predict at least as well as any other expert at least once,
and let  $\nbestset=|\bestset|\leq \nmdl$.

\begin{theorem}\label{thm:sboffline}
For any $\rate \in (0,1)$, the regret of the Soft-Bayes algorithm $\mixt$ is bounded by:
\begin{align*}
\cR_T(a) = \ln\frac{\fixtoT}{\mixttoT} \leq
  \frac{1}{\ratebar} \ln\nmdl
 + \ratebar \nbestset T
 + \nbestset\ln\frac{\nmdl}{\nbestset} + \ln\nmdl\,,
\qquad \text{where } \ratebar := \frac{\rate}{1-\rate}\,.
\end{align*}
\end{theorem}

The learning rate $\ratebar$ is optimised
by $\ratebar = \sqrt{\frac{\ln \nmdl}{T\nbestset}}$ and for this choice
\begin{align*}
\cR_T(a) &= \ln\frac{\fixtoT}{\mixttoT} \leq
 2\sqrt{T\nbestset\ln\nmdl} + \nbestset\ln\frac{\nmdl}{\nbestset} + \ln\nmdl
 \\
\text{and }\cR_T(a)&
\phantom{=\ln\frac{\fixtoT}{\mixttoT}\ }
\leq 2\sqrt{T\nmdl\ln\nmdl} + \ln\nmdl\,.
\end{align*}

To prove this theorem we need the following lemma (proof in \cref{sec:revjens}).
\begin{lemma}\label{lem:lnaixiB}
Let $\rate \in(0, 1), a\in \cW$, and $q \in [0,\infty)^N$,
\begin{align*}
\ln\sum_{i=1}^N a_i \mdlfrac_i \leq \frac{1}{\rate} \sum_{i=1}^N a_i \ln\left(1-\rate+\rate\mdlfrac_i\right) +
\max_i \ln\left(1+\frac{\rate}{1-\rate}\mdlfrac_i\right).
\end{align*}
\end{lemma}

\begin{proof}(\cref{thm:sboffline})
By Lemma~\ref{lem:lnaixiB}
and using $\wmdlitpo =\wmdlit(1-\rate)\left(1+\frac{\rate}{1-\rate}\frac{\mdlit}{\mixtt}\right)$,
for any $t \in [T]$:
\begin{align*}
\ln \frac{\fixt}{\mixtt} &=
\ln\sum_{i=1}^\nmdl \wfixi\frac{\mdlit}{\mixtt}
 \leq \frac{1}{\rate}\sum_i \wfixi\ln\left(1-\rate+\rate\frac{\mdlit}{\mixtt}\right) + \max_{i\leq\nmdl} \ln\left(1+\frac{\rate}{1-\rate}\frac{\mdlit}{\mixtt}\right)  \\
 &= \frac{1}{\rate}\sum_i \wfixi\ln\frac{\wmdlitpo}{\wmdlit} + \max_i \ln\left(\frac{\wmdlitpo}{\wmdlit}\right) - \ln(1-\rate)\,.
\end{align*}
The first term telescopes when summed over time, but more effort is required to control the second since the index $i$ of $\max_i$ can change with time. The idea is to introduce all the missing $\ln \wmdlitpo/\wmdlit$ terms
for all $i$ that can be the $\max_i$ at some step $t$, that is all $i\in\bestset$.
This comes at a cost of $\ln(1-\rate)$ for each of them since $\wmdlitpo/\wmdlit \geq 1-\rate$:
\begin{align}\label{eq:sbofflineinjectloss}
\ln \frac{\fixt}{\mixtt}
&\leq \frac{1}{\rate}\sum_{i=1}^N \wfixi\ln\frac{\wmdlitpo}{\wmdlit}
 -\nbestset\ln(1-\rate)
 + \sum_{i\in\bestset}\ln\frac{\wmdlitpo}{\wmdlit}\,.
\end{align}
We can now telescope the series over time, also using $-\ln(1-\rate)\leq \frac{\rate}{1-\rate}=\ratebar$ (Lemma~\ref{lem:-ln1-x}),
\begin{align*}
\ln \frac{\fixtoT}{\mixttoT}
&= \sum_{t=1}^T \ln \frac{\fixt}{\mixtt}
\leq  \frac{1}{\rate} \sum_{i=1}^\nmdl \wfixi\ln\frac{\wmdliTpo}{\wmdli_1}
 +\ratebar \nbestset T
 + \sum_{i\in\bestset}\ln\frac{\wmdliTpo}{\wmdli_1}\,.
\end{align*}
Using $\wmdli_1=1/\nmdl$ it holds that $\sum_i \wfixi\ln(\wmdliTpo / \wmdli_1)$ is maximised when $\wmdliTpo=\wfixi$
and similarly $\sum_{i\in\bestset}\ln(\wmdliTpo / \wmdli_1)$ is maximised when $\wmdliTpo = 1/\nbestset$. Therefore
\begin{align*}
\ln \frac{\fixtoT}{\mixttoT}
&\leq \frac{1}{\rate} \KL{\wfix}{\wmdl_1}
 +\ratebar\nbestset T
 + \nbestset\ln\frac{\nmdl}{\nbestset}
\quad\leq
 \frac{1}{\rate} \ln\nmdl
 +\ratebar\nbestset T
 + \nbestset\ln\frac{\nmdl}{\nbestset}
 \numberthis\label{eq:thm:sboffline:KL}\\
&=
 \frac{1}{\ratebar} \ln\nmdl
 +\ratebar \nbestset T
 + \nbestset\ln\frac{\nmdl}{\nbestset} +\ln\nmdl\,.
\end{align*}
If we replace $\bestset$ with all the $\nmdl$ experts in \cref{eq:sbofflineinjectloss}, we obtain $\nbestset=\nmdl$ giving the second bound.
\end{proof}

\section{Self-confident bounds}\label{sec:self-conf}

Self-confident bounds were introduced by \citet{auer2000adaptive} in online prediction
to build algorithms that can perform better when the sequence is easy,
instead of considering that all sequences are worst cases.
These bounds depend on the loss of the competitor.
For Prod, \citet{gaillard2014mlprod} derived second order self-confident bounds
that depend on the excess losses, that is, the difference between the instantaneous loss of the
learner and that of one of the experts.
We provide bounds of a similar flavour.


\begin{theorem}\label{thm:selfconfquad}
Consider the \SB{} algorithm with learning rate $\rate \in (0,1)$ and $\ratebar = \frac{\rate}{1-\rate}$. Then
\begin{align}
\cR_T(a) = \ln\frac{\fixtoT}{\mixttoT} &\leq \frac{1}{\ratebar}\ln \nmdl + \ratebar \max_{i\leq \nmdl} \sum_{t=1}^T\left(\frac{\mdlit}{\mixtt} - 1\right)^2 + \ln \nmdl \label{eq:sbofflinequadracumul} \\
&\leq \frac{1}{\ratebar}\ln \nmdl + \ratebar T \max_{i\leq\nmdl,t\leq T} \left(\frac{\mdlit}{\mixtt} - 1\right)^2 + \ln \nmdl\,. \label{eq:sbofflinequadra}
\end{align}
\end{theorem}
The proof is in \cref{sec:mainresults}.
Let $C_2 := \max_{i,t} \left(\frac{\mdlit}{\mixtt} - 1\right)^2$,
then \Cref{eq:sbofflinequadra} is optimized
for $\ratebar := \sqrt{\frac{\ln\nmdl}{T C_2}}$
leading to a regret bound of $2\sqrt{TC_2\ln\mdl} + \ln\nmdl
= \max_{i,t} 2\left(\frac{\mdlit}{\mixtt} - 1\right)\sqrt{T\ln\nmdl} + \ln\nmdl$.
Furthermore, since the learning rate is monotonically decreasing,
this bound is suitable for an online learning rate.
The case for \cref{eq:sbofflinequadracumul} is more complicated:
as noted by \citet{cesabianchi2007prod,gaillard2014mlprod} for Prod and others elsewhere this sort of bound depends on the best expert in hindsight
and thus does not lead to a monotone decreasing learning rate in general.
\citet{gaillard2014mlprod} nicely circumvent this issue by using one learning rate
per expert, at the cost of only a multiplicative $O(\ln\ln T)$ factor in the loss.
We perform a similar transformation in \cref{sec:multirate}.
%
As discussed in \cref{sec:largelosses}, since $\mdlit / \mixtt$ can be large, the quadratic dependence on $\max_{i,t} (\mdlit / \mixtt - 1)^2$ can be poor.
For this reason we show that with a small additional cost the quadratic term can be replaced with a linear term.

\begin{theorem}\label{thm:selfconflin}
Consider the \SB{} algorithm with learning rate $\rate \in (0,1)$
and let $C_1 = \sum_{t=1}^T \max_{i\leq\nmdl} \left(\frac{\mdlit}{\mixtt} - 1\right)$.
 Then
\begin{align*}
\cR_T(a) = \ln\frac{\fixtoT}{\mixttoT} \leq \min \left\{
C_1,\quad
\frac{1}{\rate}\ln \nmdl + \frac{\rate}{2} C_1  + \rate^2 T
\right\}\,.
\end{align*}
\end{theorem}
The proof is in \cref{sec:mainresults}.
If the learning rate is chosen to be
$\rate = \sqrt{\frac{2 \log N}{C_1}}$
then the theorem shows that
\begin{align*}
\cR_T(a) \leq \min\left\{C_1,\, \sqrt{2 C_1 \ln N} + \frac{2 T \ln N}{C_1}\right\}\,.
\end{align*}
Observe that if a single expert $i$ is always the best predictor, $\mixtt$ becomes close to $\mdlit$ and $\frac{\mdlit}{\mixtt} - 1$ becomes close to 0
and then $C_1 <\!< T$; However for this case learning will likely still be slower
than with second-order self-confident bounds.

In another case where for example
$\fixt = \sum_{i\in\bestset }\frac{1}{\nbestset}\mdlit$,
that is, at worst the best expert alternates uniformly
between a subset $\bestset$ of $\nbestset=|\bestset|$ experts,
$\mixtt$ cannot become close to all the $\mdlit$ at the same time
and then
$C_1 = O(T)$,
which still makes $\cR_T(a) \leq O(\sqrt{T})$ (omitting other dependencies);
Furthermore, since this means that $\mixt$ learns, $\mixtt$ should become close to $\sum_{i\in\bestset }\frac{1}{\nbestset}\mdlit$
and thus $C_1 \approx \nbestset T$
so we should have $\cR_T(a) \leq O(\sqrt{\nbestset T \ln \nmdl})$, hence possibly providing good guarantees against the best subset of the experts.
By contrast, a second-order self-confident bound would only provide
a guarantee of $O(\nbestset\sqrt{T\ln\nmdl})$.

Moreover, upper bounding $C_1$ with $T \max_{i,t}\frac{\mdlit}{\mixtt}$ in \cref{thm:selfconflin} and then optimizing $\rate$ as above
gives the result of \cref{eq:SBboundeps}
for $C = \max_{i,t}\frac{\mdlit}{\mixtt}$.

Also note that $C_1$ is monotonically increasing with $T$ and thus the learning
rate can be updated online.

\section{Online bounds}\label{sec:online}

We now provide a fully `online' algorithm that does not require advance knowledge of the time horizon $T$ or the number of `sometimes optimal experts' $\nbestset$ in advance.
Surprisingly, we could not simply replace the learning rating $\rate$ with a time-varying version $\rate_t \approx 1/\sqrt{t}$ and adapt the proofs in a straightforward manner.
The reason seems to be that with a fixed rate of $\rate := 1/\sqrt{T}$,
the weights can decay as $\wmdli_1(1-1/\sqrt{T})^t\approx \wmdli_1\exp(-t/\sqrt{T})$ after $t$ steps,
whereas for a time-varying learning rate $\ratet := 1/\sqrt{t}$,
the weights can decay as fast as $\wmdli_1\prod_{s=1}^t(1-1/\sqrt{t})\approx\wmdli_1\exp(-\sqrt{t})$ (consider for example $t=\sqrt{T}$).
An easy solution is to use the doubling trick~\citep{cesabianchi1997expert}, which was the approach taken in the analysis of the
original Prod algorithm~\citep{cesabianchi2007prod}. Another option is
to use an exponential rescaling with renormalization, which was used for ML-Prod~\citep{gaillard2014mlprod}.
We show instead a different online correction of the update rule
based on a special form of the fixed-share rule~\citep{herbster1998tracking}.
When $0<\ratetpo \leq \ratet\leq 1$, we use the following online correction term
applied to the update rule:%
\footnote{Note that this online correction can also be used with EG.}
\begin{align}\label{eq:sbonlineupdate}
\wmdlitpo := \underbrace{\wmdlit\left(1-\ratet+\ratet\frac{\mdlit}{\mixtt}\right)}_{\text{update}}\underbrace{\frac{\ratetpo}{\ratet} + \left(1-\frac{\ratetpo}{\ratet}\right)\wmdli_1}_{\text{online correction}}
\ .
\end{align}

\newcommand{\propnormal}{normalized property of \cref{eq:onlinenormal}}
\newcommand{\proprestart}{restarting property of \cref{eq:restart}}
\newcommand{\proptelescope}{telescoping property of \cref{eq:telescope}}
\newcommand{\propinjectloss}{loss injection property of \cref{eq:injectloss}}
\newcommand{\proptelescoperate}{$1/\rate$-telescoping property of \cref{eq:telescoperate}}

\begin{lemma}\label{lem:sbonlineupdate}
The online update rule of \cref{eq:sbonlineupdate} has the following properties:
\begin{align}
\wmdltpo &\in\cW
&\text{(normalized)}
\label{eq:onlinenormal}\\
\wmdlit &\geq \wmdli_1\left(1-\frac{\ratet}{\ratetmo}\right)
&\text{(restarting)}
\label{eq:restart}\\
\ln\left(1-\ratet+\ratet\frac{\mdlit}{\mixtt}\right)
&\leq
\ln \frac{\wmdlitpo}{\wmdlit} +\ln\frac{\ratet}{\ratetpo}
&\text{(telescoping)}
\label{eq:telescope}\\
0 &\leq
\ln \frac{\wmdlitpo}{\wmdlit} +\ln\frac{\ratet}{\ratetpo}
+\frac{\ratet}{1-\ratet}
&\text{(loss injection)}
\label{eq:injectloss}\\
\frac{1}{\ratet}\ln\left(1-\ratet+\ratet\frac{\mdlit}{\mixtt}\right)
&\leq \frac{1}{\ratetpo}\ln\frac{\wmdlitpo}{\wmdli_1} - \frac{1}{\ratet}\ln\frac{\wmdlit}{\wmdli_1}
&\text{($1/\rate$-telescoping})
\label{eq:telescoperate}
\end{align}
\end{lemma}
The restarting property ensures that the weights are never too small,
enabling the mixture to `restart' the learning process
and offer tracking guarantees.
The loss injection property will be useful in the theorems to force some series to telescope by `injecting' some additional loss as was done offline in the proof of \cref{thm:sboffline}.
\begin{proof}
\Cref{eq:onlinenormal} follows from the fact that the Soft-Bayes update rule keeps
the weights normalized, and so does the fixed-share rule.
Starting from \cref{eq:sbonlineupdate},
\cref{eq:restart} follows from dropping the l.h.s. of the $+$.
\Cref{eq:telescope} follows from dropping the r.h.s. of the $+$, dividing by $\wmdlit$, taking the log then rearranging.
\Cref{eq:injectloss} follows from \cref{eq:telescope} by taking $\mdlit = 0$
and from $-\ln(1-x)\leq \frac{x}{1-x}$ (Lemma~\ref{lem:-ln1-x}) and rearranging.
For \cref{eq:telescoperate}, dividing by $\wmdli_1$, taking $\beta = \frac{\ratetpo}{\ratet}$ we have:
\begin{align*}
\ln \frac{\wmdlitpo}{\wmdli_1} &= \ln\left(\beta \frac{\wmdlit}{\wmdli_1}\left(1-\ratet+\ratet\frac{\mdlit}{\mixtt} \right) + (1-\beta)\right) \\
&\geq \beta\ln\frac{\wmdlit}{\wmdli_1}\left(1-\ratet+\ratet\frac{\mdlit}{\mixtt} \right)
= \frac{\ratetpo}{\ratet}\ln\frac{\wmdlit}{\wmdli_1}\left(1-\ratet+\ratet\frac{\mdlit}{\mixtt} \right)\,,
\end{align*}
where we used Jensen's inequality, with $\beta\in[0,1]$ since $\ratetpo\geq\ratet$ as required in \cref{eq:sbonlineupdate}.
Dividing by $\ratetpo$ and rearranging gives the result.
\end{proof}

We will not provide formal results for the self-confident learner but we make the following observation.
\begin{remark}
Using the online correction rule of \cref{eq:sbonlineupdate}
with the self-confident learning rate of \cref{thm:selfconflin}
gives
\begin{align*}
\frac{\ratetpo}{\ratet} = \sqrt{\frac{\sum_{k=1}^{t-1} (\max_{i} \frac{\mdli_k}{\mixt_k}-1)}{\sum_{k=1}^t (\max_i \frac{\mdli_k}{\mixt_k}-1)}}
=\sqrt{1-\frac{\max_i \frac{\mdlit}{\mixtt}-1}{\sum_{k=1}^t (\max_i \frac{\mdli_k}{\mixt_k}-1)}}
\approx 1- \frac{\max_i \frac{\mdlit}{\mixtt}-1}{2\sum_{k=1}^t (\max_i \frac{\mdli_k}{\mixt_k}-1)}\,.
\end{align*}
Let $i_t := \arg\max_i \mdlit$.
On a step $t$ where the mixture makes a bad prediction,$\frac{\mdl^{i_t}_t}{\mixt_t}\leq1/\wmdl^{i_t}_t$ is large so the weight $\wmdl^{i_t}_t$ is small.
Considering the \proprestart, this means that
the weight $\wmdl^{i_t}_t$ (in particular) receives a boost
$1-\frac{\ratetpo}{\ratet}$
toward its prior,
hence helping the mixture coping with experts that suddenly become good predictors
after a long period of bad predictions---except that this may be one time step too late.
Indeed, for the self-confident learner the ratio $\frac{\ratetpo}{\ratet}$ can be close to 1 when the mixture predicts well, which means that the  weights of bad predictors may still decrease exponentially fast---potentially resulting in large instantaneous losses if they become good predictors later.
To prevent this, we advise replacing
$\frac{\ratetpo}{\ratet}$ with $\min \left\{\frac{\ratetpo}{\ratet}, \sqrt{\frac{t}{t+1}}\right\}$ in \cref{eq:sbonlineupdate}, which ensures that the weights do not decrease
faster than $O(1/t)$, while still retaining the quicker restarting property of the self-confident learning rate.
\end{remark}

The following generic bound will be used for the various proofs.

\begin{lemma}\label{lem:sbonlinebound}
For any sequence of monotone decreasing learning rates $\ratet\in (0, 1)$,
when using the update rule of \cref{eq:sbonlineupdate},
the regret of the mixture $\mixt$ of the $\nmdl$ experts with prior weights $\wmdl_1$ compared to the best fixed combination $\fix$ with weights $\wfix$ is bounded by:
\begin{align*}
\ln\frac{\fixtoT}{\mixttoT} \leq
\frac{1}{\rateTpo}\KL{\wfix}{\wmdl_1}
+ \ln\frac{\rateo}{\rateTpo}
+ \sum_{t=1}^T \frac{\ratet}{1-\ratet}
+ \sum_{t=1}^T \max_{i\leq\nmdl} \ln\frac{\wmdlitpo}{\wmdlit}
.
\end{align*}
\end{lemma}
\begin{proof}
Starting from Lemma~\ref{lem:lnaixiB} and similarly to the proof of \cref{thm:sboffline} we have:
\begin{align}
\ln\frac{\fixt}{\mixtt} &\leq
\sum_{i=1}^\nmdl \wfixi \frac{1}{\ratet}\ln\left(1-\ratet+\ratet\frac{\mdlit}{\mixtt}\right)
+ \max_i \ln\left(1+\frac{\ratet}{1-\ratet}\frac{\mdlit}{\mixtt}\right) \notag \\
&=\sum_{i=1}^\nmdl \wfixi \frac{1}{\ratet}\ln\left(1-\ratet+\ratet\frac{\mdlit}{\mixtt}\right)
+ \max_i \ln\left(1-\ratet+\ratet\frac{\mdlit}{\mixtt}\right) -\ln(1-\ratet)\,.
\label{eq:prelossinjection}
\end{align}
Using Lemma~\ref{lem:-ln1-x}, $-\ln(1-\ratet)\leq\frac{\ratet}{1-\ratet}$
along with the \proptelescoperate{} on the term in the sum
and the \proptelescope{} on the term in the max gives:
\begin{align*}
  \ln\frac{\fixt}{\mixtt}
&\leq
\sum_{i=1}^\nmdl \wfixi \left(\frac{1}{\ratetpo}\ln\frac{\wmdlitpo}{\wmdli_1} -\frac{1}{\ratet}\ln\frac{\wmdlit}{\wmdli_1} \right)
+ \max_i \ln\frac{\wmdlitpo}{\wmdlit}
+\ln\frac{\ratet}{\ratetpo}
+\frac{\ratet}{1-\ratet}.
\end{align*}
Summing $\ln\frac{\fixt}{\mixtt}$ over $t$ and telescoping the first term $\sum_i \wfixi(\cdot)$ and the third term (but not the second one) leads to:
\begin{align*}
\ln\frac{\fixtoT}{\mixttoT} \leq
\sum_{i=1}^\nmdl \wfixi \ln\frac{\wmdliTpo}{\wmdli_1}
+ \sum_{t=1}^T \max_i \ln\frac{\wmdlitpo}{\wmdlit}
+ \ln\frac{\rateo}{\rateTpo}
+ \sum_{t=1}^T \frac{\ratet}{1-\ratet}
.
\end{align*}
Finally, taking the worst case $\wmdliTpo = \wfixi$  and rearranging gives the result.
\end{proof}

We are now ready to show the online bounds. First we track only the time step $t$.

\begin{theorem}\label{thm:sbonlineN}
When using the update rule in \cref{eq:sbonlineupdate} with the learning rate $\ratet:=\sqrt{\frac{\ln\nmdl}{2\nmdl t}}$, we have the following regret bound against the best fixed convex combination $\fix$ of the experts:
\begin{align*}
\ln\frac{\fixtoT}{\mixttoT} \leq 2\sqrt{2(T+1)\nmdl\ln \nmdl} + (\half \nmdl + \ln\nmdl)\ln(T+1)+\ln\nmdl.
\end{align*}
\end{theorem}
\begin{proof}
We start from Lemma~\ref{lem:sbonlinebound}, and work on the $\sum_t \max_i$ term.
As in the proof of \cref{thm:sboffline}, the main idea is to make this term telescope by injecting some additional positive terms.
This is done by using the \propinjectloss{} 
 repeatedly
(starting at $t=1$)
on all the $\nmdl-1$ experts $i$ that are not already in the sum, leading to:
\begin{align*}
\sum_{t=1}^T \ln\frac{\fixt}{\mixtt}
&\leq
\frac{1}{\rateTpo}
\underbrace{\KL{\wfix}{\wmdl_1}}_{\leq \ln\nmdl}
+ \nmdl\ln\frac{\rateo}{\rateTpo}
+\nmdl\sum_{t=1}^T \underbrace{\frac{\ratet}{1-\ratet}}_{=\ratet+\frac{\ratet^2}{1-\ratet}}
+ \underbrace{\sum_{i=1}^\nmdl \ln\frac{\wmdliTpo}{\wmdli_1}}_{\leq 0} \\
&\leq \frac{1}{\rateTpo}\ln\nmdl
+\nmdl\ln\frac{\rateo}{\rateTpo}
+ \nmdl\sum_{t=1}^T \ratet
+ \nmdl\sum_{t=1}^T \frac{\ratet^2}{1-\ratet}.
\intertext{Taking $\ratet := \sqrt{\frac{\ln\nmdl}{2\nmdl t}}$,
and since $1-\ratet \geq \half$, we obtain:}
\sum_{t=1}^T \ln\frac{\fixt}{\mixtt}
&\leq
\sqrt{2(T+1)\nmdl\ln\nmdl} + \nmdl \ln \sqrt{T+1} +
\sqrt{\frac{\nmdl\ln\nmdl}{2}}
\underbrace{\sum_{t=1}^T \frac{1}{\sqrt{t}}}_{\leq 2\sqrt{T+1}}
 +
\ln\nmdl\underbrace{\sum_{t=1}^T \frac{1}{t}}_{\leq 1+\ln T} \\
&\leq 2\sqrt{2(T+1)\nmdl\ln \nmdl} + (\half \nmdl + \ln\nmdl)\ln(T+1)+\ln\nmdl
.
\end{align*}
\end{proof}

This regret is only a factor $\sqrt{2}\approx 1.41$ worse on the leading term than the corresponding offline bound using $\rate = \sqrt{\ln(\nmdl) / (\nmdl T)}$,
which is better than the $\sqrt{2}/ (\sqrt{2}-1)\approx 3.41$ factor that would be obtained via the doubling trick.

\subsection{Sparse expert set\texorpdfstring{: Tracking $\bestsett$}{}}

We would like to have an online bound of the order $O(\sqrt{T\nbestset\ln\nmdl})$ as in \cref{thm:sboffline} where $\nbestset$ is the number of `good' experts.
Interestingly, setting naively $\ratet = \sqrt{\ln\nmdl / (2t\nbestsett )}$
works, but not for the naive reasons.
Indeed, merely adapting the proof of \cref{thm:sboffline} leads to $\sum_{i\in\bestset}\sum_t \ratet \approx \nbestset\sqrt{T\ln\nmdl}$ regret in the worst case where only one expert is good for $T-\nbestset+1$ steps ($\nbestsett= 1$),
and on the $\nbestset-1$ last steps the other experts are the best predictors.
Let $\Tstarti$ the first time step at which expert $i$ is the best expert,%
\footnote{Breaking ties can be done most favourably by picking
as the best expert one that was already counted as such, to avoid introducing new `good' experts.}
that is $\Tstarti := \min \{t:\mdlit =\max_j \mdljt\}$.
Let $\bestsett := \{i:\Tstarti < t\}$ be the set of experts
that have been the best expert at least once (strictly) before time step $t$,
let $\nbestsett := \max \{1, |\bestsett| \}$,
$\bestset := \bestsetTpo$
and $\nbestset := \nbestsetTpo$.

\begin{theorem}\label{thm:sparse}
When using the update rule in \cref{eq:sbonlineupdate} with the learning rate $\ratet:=\sqrt{\frac{\ln\nmdl}{2\nbestsett t}}$, we have the following regret bound against the best fixed convex combination $\fix$ of the experts:
\begin{align*}
\ln\frac{\fixtoT}{\mixttoT}
&\leq 2\sqrt{2\nbestset(T+1)\ln\nmdl}
+ \left(\nbestset+\ln\nmdl\right)\ln T
+ \nbestset\ln\frac{\nmdl}{\nbestset} \\
&\quad + 1.2\nbestset
+ \sqrt{\half\ln\nmdl}(1+\ln\nbestset) + 3.5\ln\nmdl.
\end{align*}
\end{theorem}
The proof is in \cref{sec:mainresults}.
Again, we only get a $\sqrt{2}$ factor on the leading term compared to the offline version.
One drawback of this algorithm is that any expert that is the best one even only once will be counted in $\nbestset$.
It may be desirable to forget about experts that have not been best
for a long time and thus decrease $\nbestset$.
This is left as an open problem.

\subsection{Shifting regret}
In this subsection we show that the online version of Prod can compete with the best sequence of convex combinations of experts.
Let $1 \leq K \leq T$ and $\fix$ be a sequence of constant competitors
$\fixidx{1}, \fixidx{2}\ldots\fixidx{\nshift}$.
Each competitor $\fixk$ starts at step $\Tshiftk$ and ends at step $\Tshiftkend=\Tshiftkpo-1$.
Thus, assuming $\Tshiftidxend{\nshift}=T$, we have
$\fixtoT = \fixidx{1}_{1:\Tshiftidxend{1}}\fixidx{2}_{\Tshiftidx{2}:\Tshiftidx{2}}\ldots\fixidx{\nshift}_{\Tshiftidx{\nshift}:T}$.
Each competitor $\fixk$ has associated weights $\wfixik$ that remain fixed
on the interval $\Tshiftk:\Tshiftkend$.


With the learning rate of \cref{thm:sbonlineN}, we readily obtain a shifting
regret bounded in $O(\nshift\ln(T)\sqrt{T\nmdl\ln\nmdl})$, but by tuning the learning rate
and still without prior knowledge of $\nshift$
it can be reduced to $O((\nshift+\ln T)\sqrt{T\nmdl\ln\nmdl})$:
\begin{theorem}\label{thm:shifting}
When using
the learning rate $\ratet = \sqrt{\frac{\ln\nmdl}{2\nmdl t}}\ln(t+3)$
with the update rule in \cref{eq:sbonlineupdate},
 for $T\geq 2$ the K-shifting regret is bounded by:
\begin{align*}
\ln\frac{\fixtoT}{\mixttoT}
&\leq
\sqrt{2(T+1)\nmdl\ln\nmdl}\left(
\ln(T+3) + \nshift\left(\frac{2}{\ln \nmdl}+\frac{1}{\ln T}\right)
\right) \\
&\quad
+ \frac{5}{4}\frac{\ln\nmdl}{\nmdl}(1+\ln T)^3
+ \frac{\nmdl}{2}\ln(T+1)\,.
\end{align*}
\end{theorem}

If $\nshift$ were known in advance, then the learning rate can be tuned so that the regret is $O(\sqrt{T\nshift\nmdl \ln (\nmdl T)})$.
 In the absence of this knowledge
one can still have the $\nshift$ inside the square root by competing with several learning rates as discussed in the conclusion.
The proof is in \cref{sec:mainresults}.


\section{Conclusion}\label{sec:conc}

We have shown new regret guarantees for the Prod algorithm when competing against a convex combination of the experts.
In particular, we proved that $\cR_T = O(\sqrt{T\nmdl\ln\nmdl})$, which unlike EG does not depend on the (possibly unbounded) largest gradient.
The online version of the algorithm uses a special form of the fixed-share rule, which simultaneously makes the algorithm truly online, computable in $O(\nmdl)$ steps per round and also
enjoys strong shifting regret guarantees.
A short discussion and some open questions follow.

\paragraph{Alternative approaches}
As discussed in Section~\ref{sec:largelosses}, the EG algorithm fails when
it encounters large gradients. Since these only occur when the weights are close to zero, one might try an approach based on follow-the-regularised-leader
or mirror descent \citep[for an overview]{Haz16}. The natural choice of regulariser is $R(w) = -\sum_i \log w^i$, which after a long calculation can be shown to
eliminate the dependence on the largest gradient. The regret guarantee is slightly worse than for Prod, as it has a logarithmic
dependence on $T$ rather than $N$ so that $\cR_T = O(\sqrt{NT \log(T/N)})$. Furthermore, the algorithm does not immediately have tracking guarantees and
the projection step involves a line-search, which naively requires a computation time of $O(N \log(T))$ per round.

\paragraph{Mixtures of learning rates}
Many of the results in the previous sections have depended on a specific `optimal' choice of learning rate that allows
Prod or its online variant to adapt to specific kinds of structure in the data. The downside of fixing a single learning
rate is that if the structure of interest is not present, then the algorithm may perform badly. In many cases it is possible
to derive a clever scheme for adapting the learning rate online to achieve the best of several worlds. An alternative is to exploit
the special property of the log-loss and to simply create a meta-agent that mixes over a discrete set of predictors.
For example, in the offline case one can predict using the Bayesian mixture over a set of Soft-Bayes predictors
with learning rates $(\rate_i)_{i=0}^K$ where $\rate_i = 2^{-i}$ and $K = \log_2(T)$. If the uniform prior is used, then
the Bayesian mixture will suffer an additional regret of only $\ln(K) = \ln \log_2(T)$ relative to the best soft-Bayes in the class.
Provided that the optimal learning rate lies in $[1/T, 1]$, then this mixture will compete with a learning rate that is at most a factor of $2$ off for
which the penalty is just a factor of $2$ at worst.
The additional regret is small enough to be insignificant. More concerning is that the computation cost becomes $O(KN)$ per round. In practice, however, this procedure is easily parallelised
and often leads to significant improvement.
Additionally, at almost no additional computation cost we can build a switching mixture \citep{herbster1998tracking,veness2012context}
between the individual experts (full Bayes, $\rate = 1$)
and a Soft-Bayes mixture with rate $\sqrt{\ln\nmdl / (2\nmdl t)}$,
to enjoy logarithmic loss ($\ln(NT)$, the cost of switching) against segments of the sequence where a single expert
is the best one, and revert to $\sqrt{T}$ loss for segments where
we need to compete against a combination of the experts.
The computation time is still in $O(\nmdl)$ per round.

\paragraph{Computationally efficient logarithmic regret}
The holy grail would be an $O(N)$-time algorithm with logarithmic regret and no dependence on the largest observed gradients of the (linearised) loss.
A reasonable conjecture is that this is not possible, a proof of which would be quite remarkable.
The most efficient algorithm with logarithmic regret is
online Newton step for which the best known computation time is $O(N^2)$, which is prohibitively large for $n \gtrsim 10^4$.
Furthermore, the online Newton step suffers from the same catastrophic failures as EG when poorly performing predictors suddenly become good. This limitation can
be overcome via additional regularisation as for mirror descent above, but naively this pushes the computation cost to at least $O(N^4)$ because the projection step becomes
more complex.

\paragraph{Different frameworks}
Another interesting direction is to extend the analysis beyond the log-loss and the linear mixing.
Regarding losses, the most natural first step might be to examine the exp-concave case.
Alternatively one could generalise the linear mixture to (say) a geometric mixture and see if Prod or similar can be applied to this practical setting \citep{mattern2016phd}.

\newcommand{\thankyouall}{We would like to thank the following people for their help:
Marc Bellemare, Guillaume Desjardins, Marc Lanctot, Andrew Lefranq, Jan Leike, R\'emi Munos, Georg Ostrovski, Bernardo Avila Pires, David Saxton, Joel Veness.}

\ifarxiv
    \paragraph{Acknowledgements}\thankyouall
\fi


\newpage

\crefname{appendix}{appendix}{appendices}
\Crefname{appendix}{Appendix}{Appendices}

\appendix

\section{Technical results}\label[appendix]{sec:technical}

\begin{lemma}\label{lem:-ln1-x}
\begin{align*}
\forall x<1: -\ln(1-x) \leq \frac{x}{1-x}.
\end{align*}
\end{lemma}
\begin{proof}
\begin{align*}
-\ln(1-x) = \ln\left(\frac{1}{1-x}\right) = \ln\left(1+\frac{x}{1-x}\right) \leq \frac{x}{1-x}
\end{align*}
where the inequality follows from $\ln(1+x)\leq x$.
\end{proof}

\begin{lemma}[\citealp{love1980log}]\label{lem:ln1+x_lowbound}
\[
\forall x\geq 0: \ln(1+x) \geq \left(\frac{1}{x}+\frac{1}{2}\right)^{-1} =
\frac{x}{1+x/2} = \frac{2x}{2+x}.
\]
\end{lemma}
\begin{proof}
\begin{align*}
\text{Let } f(x) &:= (2+x)\ln(1+x) -2x \\
\text{then } f'(x) &= \ln(1+x) + \frac{2+x}{1+x} - 2 = \ln(1+x) + \frac{1}{1+x} - 1 \\
\text{and } f''(x) &= \frac{1}{1+x} -\frac{1}{(1+x)^2} = \frac{x}{(1+x)^2}.
\end{align*}
For all $x \geq 0$, since $f''(x)\geq 0$, $f$ is convex, and since $f(0)=0$ and
$f'(0) = 0$
then $f(x) \geq 0$, which proves the result.
\end{proof}

\begin{lemma}\label{lem:ln1+x_upbound}
\[
\forall x \geq 0:\quad \ln(1+x) \leq x - \frac{x^2/2}{1+x}.
\]
\end{lemma}
\begin{proof}
Let $f(x) := x - \frac{x^2/2}{1+x} - \ln(1+x)$.
Then
\begin{align*}
f'(x) = 1 - \frac{x}{1+x} + \frac{x^2/2}{(1+x)^2} - \frac{1}{1+x}
= \frac{x^2/2}{(1+x)^2}.
\end{align*}
Since $f(0)=0$, and $f'(x)>0\ \forall x >0$,
$f$ is positive monotone increasing, which proves the result.
\end{proof}

\begin{corollary}\label{cor:-1/xln1-x}
\[
\forall x\in(0,\half]:\quad \frac{1}{x}\ln\frac{1}{1-x} -1\leq x/2 + x^2.
\]
\end{corollary}
\begin{proof}
Using Lemma~\ref{lem:ln1+x_upbound} and $\frac{1}{1-x} = 1+ \frac{x}{1-x}$:
\begin{align*}
\ln\frac{1}{1-x} &= \ln\left(1+\frac{x}{1-x}\right) \leq \frac{x}{1-x}
- \half\frac{\frac{x^2}{(1-x)^2}}{1+\frac{x}{1-x}}
= \frac{x}{1-x} - \half\frac{x^2}{1-x} \\
&= x + \frac{x^2}{1-x}- \half\frac{x^2}{1-x} \\
&= x + \frac{x^2/2}{1-x}.
\intertext{Hence}
\frac{1}{x}\ln\frac{1}{1-x} -1 &\leq \frac{x/2}{1-x}
= x/2+\frac{x^2/2}{1-x} \leq x/2+x^2
\end{align*}
where the last inequality holds if $x\leq \half$.
\end{proof}

\begin{lemma}\label{lem:lnratesimple}
\begin{align*}
\forall x \geq 0, \forall \rate\in(0, 1) :\quad
(x-1) \leq
\frac{1}{\rate}\ln (1-\rate+\rate x)
+\frac{\rate}{1-\rate}(x-1)^2.
\end{align*}
\end{lemma}
\begin{proof}
Using $\log (1+x) \geq \frac{x}{1+x}$:
\begin{align*}
\frac{1}{\rate}\ln (1-\rate+\rate x) &=\frac{1}{\rate}\ln (1+\rate(x-1))
\geq \frac{x-1}{1+\rate(x-1)}
= (x-1) - \frac{\rate(x-1)^2}{1+\rate(x-1)}  \\
&\geq(x-1) - \frac{\rate(x-1)^2}{1-\rate}
\end{align*}
where the last inequality holds with $x\geq 0$.
Rearranging gives the result.
\end{proof}

\newcommand{\dd}{3}
\newcommand{\ddmo}{2}
\begin{lemma}\label{lem:restartlnt}
For all $t=1, 2, 3,\ldots$:
\begin{align*}
-\ln\left(1-\frac{\ln(t+\dd)}{\ln(t+\ddmo)}\sqrt{\frac{t-1}{t}}\right)
\leq \ln(t) + 1.6.
\end{align*}
\end{lemma}
\begin{proof}
By exhaustive search, the result holds for all $t\in [1..30]$.
Now consider $t\geq 30$ for the rest of the proof.
Observe that
\begin{align*}
\frac{\ln(t+3)}{\ln(t+2)}=1+\frac{\ln(1+1/(t+2))}{\ln(t+2)}\leq 1+\frac{1}{(t+2)\ln(t+2)}\,.
\end{align*}
Therefore:
\begin{align*}
-\ln&\left(1-\frac{\ln(t+\dd)}{\ln(t+\ddmo)}\sqrt{\frac{t-1}{t}}\right)
 \leq\ln(t) - \ln \left(t - \left(1+\frac{1}{(t+\ddmo)\ln(t+\ddmo)}\right)\sqrt{t(t-1)} \right) \\
 &= \ln(t) - \ln\left( t(1-\underbrace{\sqrt{1-1/t}}_{\leq 1-1/(2t)}) - \frac{\sqrt{t(t-1)}}{(t+\ddmo)\ln(t+\ddmo)}\right) \\
 &\leq \ln(t) - \ln\left(\half - \frac{1}{\ln(t+\ddmo)}\right)
 \quad\leq \ln(t) - \ln\left(\half - \frac{1}{\ln(30+\ddmo)}\right)
 \quad\leq \ln(t) + 1.6
 \,.
\end{align*}
\end{proof}


\section{Reverse Jensens' inequalities}\label[appendix]{sec:revjens}

\begin{lemma}\label{lem:lnaixi}
Let $\rate \in(0, \half], a\in\cW$, and $\forall i\in[\nmdl]: \mdlfrac_i \geq 0$ then:
\begin{align*}
\ln\sum_{i=1}^\nmdl a_i \mdlfrac_i \leq \sum_{i=1}^\nmdl a_i \frac{1}{\rate}\ln\left(1-\rate+\rate\mdlfrac_i\right) +
\max_{i\leq\nmdl} \frac{\rate}{2}\left(\mdlfrac_i-1\right)+\rate^2.
\end{align*}
\end{lemma}
\begin{proof}
Let $\lnrate(\mdlfrac) := \frac{1}{\rate}\ln(1-\rate+\rate \mdlfrac)$, and let $\ratebar := \frac{\rate}{1-\rate}$.
By concavity of $\lnrate$, for $\mdlfrac\in[0, \mdlfracmax]$:
\begin{align*}
\lnrate(\mdlfrac) &\geq \lnrate(0) + \frac{\mdlfrac}{\mdlfracmax}\left(\lnrate(\mdlfracmax)-\lnrate(0)\right)  \\
&= \frac{1}{\rate}\ln(1-\rate) + \frac{\mdlfrac}{\mdlfracmax}\left(
\frac{1}{\rate}\ln\left[(1-\rate)\left(1+\frac{\rate}{1-\rate}\mdlfracmax\right)\right]
- \frac{1}{\rate}\ln(1-\rate)
\right)\\
&= \frac{1}{\rate}\ln(1-\rate) + \mdlfrac\frac{\ln(1+\ratebar \mdlfracmax)}{\rate \mdlfracmax}
\quad =: g(\mdlfrac).
\end{align*}
Now since $\frac{\D}{\D \mdlfrac} (\ln \mdlfrac - g(\mdlfrac)) = \frac{1}{\mdlfrac} - \frac{\ln(1+\ratebar \mdlfracmax)}{\rate \mdlfracmax}$,
the maximum of $\ln \mdlfrac - g(\mdlfrac)$ is found at $\hat{\mdlfrac} = \frac{\rate \mdlfracmax}{\ln(1+\ratebar \mdlfracmax)}$.
Therefore:
\begin{align}\label{eq:forlnaixiB}
\ln \mdlfrac - g(\mdlfrac) \leq
\ln \hat{\mdlfrac} - g(\hat{\mdlfrac})
&\leq \ln\frac{\rate \mdlfracmax}{\ln(1+\ratebar \mdlfracmax)} - \frac{1}{\rate}\ln(1-\rate) - 1.
\end{align}
Using Lemma~\ref{lem:ln1+x_lowbound}, $\ln(1+\ratebar\mdlfracmax)\geq \frac{\ratebar\mdlfracmax}{1+\ratebar\mdlfracmax/2}=\frac{\rate\mdlfracmax}{1-\rate+\rate\mdlfracmax/2}$ and thus:
\begin{align*}
\ln \mdlfrac - g(\mdlfrac) &\leq
\ln(1-\rate+\rate\mdlfracmax/2) - \frac{1}{\rate}\ln(1-\rate) - 1 \\
&=
\ln(1+\rate(\mdlfracmax/2-1)) - \frac{1}{\rate}\ln(1-\rate) - 1.
\intertext{Hence, using \cref{cor:-1/xln1-x} with $\rate\leq \half$,
together with $\ln(1+x)\leq x$:}
\ln \mdlfrac - g(\mdlfrac) &\leq \frac{\rate}{2}(\mdlfracmax-1) + \rate^2.
\end{align*}
Finally, by linearity of $g$ we have $g(\sum_i a_i \mdlfrac_i) = \sum_i a_i g(\mdlfrac_i)$ and thus since $\lnrate(\mdlfrac)\geq g(\mdlfrac)$:
\begin{align*}
\ln\sum_i a_i \mdlfrac_i \leq \sum_i a_i g(\mdlfrac_i) + \frac{\rate}{2}(\mdlfracmax-1)+\rate^2
\leq \sum_i a_i\lnrate(\mdlfrac_i)+\frac{\rate}{2}(\mdlfracmax-1)+\rate^2.
\end{align*}
Substituting $\lnrate$ and $\mdlfracmax$ by their definitions finishes the proof.
\end{proof}

\begin{proof}(Lemma~\ref{lem:lnaixiB})
The beginning of the proof matches that of Lemma~\ref{lem:lnaixi}, and thus we start from \cref{eq:forlnaixiB}.
Since $\ln(1+\ratebar \mdlfracmax)\geq \frac{\ratebar \mdlfracmax}{1+\ratebar\mdlfracmax}=\frac{\rate \mdlfracmax}{1-\rate + \rate\mdlfracmax}$:
\begin{align*}
\ln \hat{\mdlfrac} - g(\hat{\mdlfrac}) &\leq \ln(1-\rate+\rate \mdlfracmax) -\frac{1}{\rate}\ln (1-\rate) -1 \\
& \leq \ln\left(1+\frac{\rate}{1-\rate} \mdlfracmax\right) -\underbrace{\left(\frac{1}{\rate}-1\right)}_{=\frac{1-\rate}{\rate}}\ln (1-\rate) -1 
\intertext{and since from Lemma~\ref{lem:-ln1-x} $-\ln(1-\rate) \leq \frac{\rate}{1-\rate}=\ratebar$:}
\ln \hat{\mdlfrac} - g(\hat{\mdlfrac}) &\leq \ln(1+\ratebar \mdlfracmax).
\end{align*}
By linearity of $g$, we have $g(\sum_i a_i \mdlfrac_i) = \sum_i a_i g(\mdlfrac_i)$ and thus:
\begin{align*}
\ln \sum_i a_i \mdlfrac_i &\leq \sum_i a_i g(\mdlfrac_i) + \ln(1+\ratebar \mdlfracmax)
\leq \sum_i a_i \ln_\rate(\mdlfrac_i)+ \ln(1+\ratebar \mdlfracmax)
\end{align*}
where the last inequality follows from $\lnrate(q) \geq g(q)$.
\end{proof}


\section{Proofs for the main results}\label[appendix]{sec:mainresults}

\begin{proof}(\cref{thm:selfconfquad})
Using $\ln(x)\leq x-1$ followed by
Lemma~\ref{lem:lnratesimple} and the definition of the update rule \cref{eq:update} leads to:
\begin{align*}
\sum_{t=1}^T \ln\frac{\fixt}{\mixtt} &\leq \sum_t \left(\frac{\fixt}{\mixtt}-1\right)
= \sum_t \sum_{i=1}^\nmdl \wfixi\left(\frac{\mdlit}{\mixtt}-1\right) \\
&\leq\sum_t \sum_i \wfixi \left(\frac{1}{\rate}\ln\left(1-\rate+\rate\frac{\mdlit}{\mixtt}\right) + \frac{\rate}{1-\rate}\left(\frac{\mdlit}{\mixtt}-1\right)^2\right) \\
&= \sum_t \sum_i \wfixi\left(\frac{1}{\rate}\ln\frac{\wmdlitpo}{\wmdlit}
+ \frac{\rate}{1-\rate}\left(\frac{\mdlit}{\mixtt}-1\right)^2\right) \\
&= \frac{1}{\rate}\sum_i \wfixi \ln\frac{\wmdliTpo}{\wmdli_1}
+ \frac{\rate}{1-\rate}\sum_i \wfixi \sum_t \left(\frac{\mdlit}{\mixtt}-1\right)^2 \\
&\leq \frac{1}{\rate}\ln\nmdl
+ \frac{\rate}{1-\rate}\max_i \sum_t \left(\frac{\mdlit}{\mixtt}-1\right)^2  \\
&=
 \frac{1-\rate}{\rate}\ln\nmdl
+ \frac{\rate}{1-\rate}\max_i \sum_t \left(\frac{\mdlit}{\mixtt}-1\right)^2
+\ln\nmdl\,.
\end{align*}
The result is completed by substituting the definitions.
\end{proof}

\begin{proof}(\cref{thm:selfconflin})
For the first entry in the minimum we use the fact that $\log x \leq x - 1$:
\begin{align*}
\sum_{t=1}^T \ln\frac{\fixt}{\mixtt}  &\leq \sum_t \left(\frac{\fixt}{\mixt} - 1\right)
= \sum_t \sum_{i=1}^\nmdl \wfixi\left(\frac{\mdlit}{\mixtt}-1\right)
\leq \sum_t \max_i\left(\frac{\mdlit}{\mixtt}-1\right)\,.
\end{align*}
For the other terms we use Lemma~\ref{lem:lnaixi} to obtain:
\begin{align*}
\ln\frac{\fixt}{\mixtt}
&=\ln\left(\sum_i \wfixi\frac{\mdlit}{\mixtt}\right)
\leq \sum_i \wfixi\frac{1}{\rate}\ln\left(1-\rate+\rate\frac{\mdlit}{\mixtt}\right) + \frac{\rate}{2}\max_i \left(\frac{\mdlit}{\mixtt}-1\right) + \rate^2 \\
&= \sum_i \wfixi \frac{1}{\rate}\ln\frac{\wmdlitpo}{\wmdlit}
+ \frac{\rate}{2}\max_i \left(\frac{\mdlit}{\mixtt}-1\right) + \rate^2\,.
\end{align*}
Therefore
\begin{align*}
\sum_{t=1}^T\ln\frac{\fixt}{\mixtt}
&\leq  \frac{1}{\rate}\ln \nmdl
+ \frac{\rate}{2}\sum_{t=1}^T\max_i \left(\frac{\mdlit}{\mixtt}-1\right) + \rate^2T.
\end{align*}
\end{proof}

\begin{proof}(\cref{thm:sparse})
As for the proof of \cref{thm:sbonlineN} we start from Lemma~\ref{lem:sbonlinebound}:
\begin{align*}
\ln\frac{\fixtoT}{\mixttoT} &\leq
\frac{1}{\rateTpo}\KL{\wfix}{\wmdl_1}
+ \ln\frac{\rateo}{\rateTpo}
+ \sum_{t=1}^T \frac{\ratet}{1-\ratet}
+ \sum_{t=1}^T \max_{i\leq\nmdl} \ln\frac{\wmdlitpo}{\wmdlit}
\intertext{but instead of complementing the missing terms for each expert from $t=1$ to $T$ (which would lead to $O(\nbestset\sqrt{T})$ if all $\Tstarti\approx T$)
we complement using the \propinjectloss{} only for expert $i$ from $\Tstarti$ to $T$ and rely on the \proprestart{}
to start with a high enough weight $\wmdliTstart \approx 1/t$.
Telescoping the series, we obtain:}
\ln\frac{\fixtoT}{\mixttoT} &\leq
\frac{1}{\rateTpo}\ln\nmdl +
\sum_{i\in\bestset} \left[
\ln\frac{\wmdliTpo}{\wmdliTstart}
+ \ln\frac{\rateTstarti}{\rateT}
+ \sum_{t=\Tstarti}^T \frac{\ratet}{1-\ratet}
\right].
\intertext{Using the \proprestart:}
\sum_{i\in\bestset} \ln\frac{\wmdliTpo}{\wmdliTstart}
&\leq \sum_{i\in\bestset} \ln \frac{\wmdliTpo}{\wmdli_1\left(1-\frac{\rateTstarti}{\rateTstartimo}\right)}
= \sum_{i\in\bestset}\left[ \ln \frac{\wmdliTpo}{\wmdli_1} - \ln\left(1-\frac{\rateTstarti}{\rateTstartimo}\right)\right] \\
&\leq \nbestset\ln\frac{\nmdl}{\nbestset}
- \sum_{i\in\bestset}\ln\left(1-\frac{\rateTstarti}{\rateTstartimo}\right)\,,\\
- \sum_{i\in\bestset}\ln\left(1-\frac{\rateTstarti}{\rateTstartimo}\right)
&=
- \sum_{i\in\bestset}\ln\left(1-\sqrt{\frac{\nbestsetTstartimo (\Tstarti-1)}{\nbestsetTstarti \Tstarti}}\right) \\
&\leq
- \sum_{i\in\bestset}\ln\left(1-\sqrt{\frac{\Tstarti-1}{\Tstarti}}\right)
=
- \sum_{i\in\bestset}\ln\left(1-\sqrt{1-\frac{1}{\Tstarti}}\right) \\
&\leq
\sum_{i\in\bestset}\ln\left(2\Tstarti\right) = \nbestset\ln2 + \sum_{i\in\bestset}\ln \Tstarti
\\
\intertext{where we used $\sqrt{1-x} \leq 1-x/2$.
For the next term, remember that $\bestsetTstarti$ does not yet include the
expert $i$, which will be added at $\Tstarti+1$:}
\sum_{i\in\bestset}\ln\frac{\rateTstarti}{\rateT}
&\leq \sum_{i\in\bestset}\ln \sqrt{\frac{\nbestset T}{\nbestsetTstarti \Tstarti}}
= \half(\nbestset\ln\nbestset - \ln ((\nbestset-1)!) + \half\sum_{i\in\bestset} \ln\frac{T}{\Tstarti} \\
&\leq \half(\nbestset + \ln\nbestset)
+ \frac{\nbestset}{2}\ln T - \half\sum_{i\in\bestset}\ln\Tstarti
\intertext{where we used $\ln((m-1)!) = \ln(m!) - \ln m$ and $\ln(m!) \geq m\ln(m) - m$.
For the next term, we have:}
\sum_{i\in\bestset}\sum_{t=\Tstarti}^T \frac{\ratet}{1-\ratet}
&= \sum_{i\in\bestset}\sum_{t=\Tstarti}^T \Big(\ratet
+ \underbrace{\frac{\ratet^2}{1-\ratet}}_{\leq 2\ratet^2}\Big)
= \sum_{t=1}^T \nbestsettpo (\ratet + 2\ratet^2) \\
&= \sum_{t=1}^T \nbestsett (\ratet + 2\ratet^2) +
\sum_{i\in\bestset} \left(\rateTstarti + 2\rateTstarti^2\right).
\\
\intertext{For the rightmost term, we take the worst case of largest learning rates,
that is for $\nbestsett = t$ up to $t=m$:}
\sum_{i\in\bestset} \rateTstarti + 2\rateTstarti^2
&\leq
\sum_{t=1}^\nbestset
\sqrt{\frac{\ln\nmdl}{2t^2}}
+ \frac{\ln\nmdl}{t^2} \leq \sqrt{\half\ln\nmdl} (1+\ln\nbestset) + 2\ln\nmdl\,,
\\
\sum_{t=1}^T \nbestsett (\ratet + 2\ratet^2)
&= \sum_{t=1}^T\sqrt{\frac{\nbestsett\ln\nmdl}{2t}} + \ln\nmdl\frac{1}{t}
\leq \sum_{t=1}^T\sqrt{\frac{\nbestset\ln\nmdl}{2t}} + \ln\nmdl\frac{1}{t} \\
&\leq \sqrt{2\nbestset(T+1)\ln\nmdl} + (1+\ln T)\ln N.
\end{align*}
Putting it all together we have:
\begin{align*}
\ln\frac{\fixtoT}{\mixttoT}
&\leq
2\sqrt{2\nbestset(T+1)\ln\nmdl}
+ \left(\frac{\nbestset}{2}+\ln\nmdl\right)\ln T
+ \half\sum_{i\in\bestset}\ln \Tstarti
+ \nbestset\ln\frac{\nmdl}{\nbestset} \\
&\quad
+ \nbestset\ln 2
+ \half\nbestset +\half \ln\nbestset
+ \sqrt{\half\ln\nmdl}(1+\ln\nbestset) + 3\ln\nmdl \\
&\leq
2\sqrt{2\nbestset(T+1)\ln\nmdl}
+ \left(\nbestset+\ln\nmdl\right)\ln T
+ \nbestset\ln\frac{\nmdl}{\nbestset} \\
&\quad + 1.2\nbestset
+ \sqrt{\half\ln\nmdl}(1+\ln\nbestset) + 3.5\ln\nmdl
\end{align*}
which concludes the proof.
\end{proof}

\begin{proof}(\cref{thm:shifting})
The beginning of the proof is similar to that of \cref{thm:sbonlineN},
using first Lemma~\ref{lem:lnaixiB}:
\begin{align*}
\ln\frac{\fixtoT}{\mixttoT}
&= \sum_{\idxshift=1}^{\nshift} \ln\frac{\fixidx{\idxshift}_{\Tshiftk:\Tshiftkend}}{\mixt_{\Tshiftk:\Tshiftkend}} \\
&\leq \sum_{\idxshift=1}^{\nshift} \sum_{t=\Tshiftk}^{\Tshiftkend}
\left[\sum_{i=1}^\nmdl \wfixik\frac{1}{\ratet}\ln\left(1-\ratet+\ratet\frac{\mdlit}{\mixtt}\right)\right] + \max_i \ln\left(1+\ratebart\frac{\mdlit}{\mixtt}\right) \\
&=
\underbrace{\sum_{t=1}^T \max_i \ln\left(1+\ratebart\frac{\mdlit}{\mixtt}\right)}_{(A)}
+ \sum_{\idxshift=1}^{\nshift}
\underbrace{\sum_{t=\Tshiftk}^{\Tshiftkend}
\sum_{i=1}^\nmdl \wfixik\frac{1}{\ratet}\ln\left(1-\ratet+\ratet\frac{\mdlit}{\mixtt}\right)}_{(B)}.
\end{align*}
For (A), we first apply the same transformation as in \cref{eq:prelossinjection},
then we repeatedly
 use the \propinjectloss{} and
 the \proptelescope{} as in the proof of \cref{thm:sbonlineN}.
It can be shown that $\max_{N\geq 2, t\geq 1} \ratet \leq 3/5$, hence:
\begin{align*}
(A)
&\leq
\nmdl\ln\frac{\rateo}{\rateTpo}+
\nmdl\sum_{t=1}^T\left( \ratet + \frac{\ratet^2}{1-\ratet}\right) \\
&\leq
N\ln\frac{\ln(4)\sqrt{T+1}}{\ln(T+4)}+
\sqrt{\half\nmdl\ln\nmdl}\ln(T+3)\sum_{t=1}^T \frac{1}{\sqrt{t}}
+ \frac{5}{2}\frac{\ln\nmdl}{2\nmdl}(\ln(T+3))^2\sum_{t=1}^T \frac{1}{t} \\
&\leq
\frac{N}{2}\ln(T+1)+\sqrt{2\nmdl\ln\nmdl}\ln(T+3)\sqrt{T}+
\frac{5}{4}\frac{\ln\nmdl}{\nmdl}(\underbrace{\ln(T+3)}_{\leq 1+ \ln T, \forall T\geq 2})^2(1+\ln T) \\
&\leq
\frac{N}{2}\ln(T+1)+
\sqrt{2T\nmdl\ln\nmdl}\ln(T+3)+
\frac{5}{4}\frac{\ln\nmdl}{\nmdl}(1+\ln T)^3
\end{align*}
For (B), using the \proptelescoperate{} and then telescoping the series gives:
\begin{align*}
(B)&\leq
\sum_{i=1}^\nmdl \wfixik\left(
\frac{1}{\rate_{\Tshiftkpo}}\ln\frac{\wmdli_{\Tshiftkpo}}{\wmdli_1} - \frac{1}{\rate_{\Tshiftk}}\ln\frac{\wmdli_{\Tshiftk}}{\wmdli_1}\right)\,.
\end{align*}
Now, with the \proprestart{}, $\ln\frac{\wmdli_{\Tshiftk}}{\wmdli_1} \geq \ln\left(1-\frac{\rate_{\Tshiftk}}{\rate_{\Tshiftk-1}}\right)$
together with Lemma~\ref{lem:restartlnt},
and also with $\sum_{i=1}^\nmdl \wfixik\ln\frac{\wmdli_{\Tshiftkpo}}{\wmdli_1} \leq \ln\nmdl$ we have:
\begin{align*}
(B) &\leq
 \frac{1}{\rate_{\Tshiftkpo}}\ln\nmdl
-\frac{1}{\rate_{\Tshiftk}}\ln\left(1-\frac{\rate_{\Tshiftk}}{\rate_{\Tshiftk-1}}\right)
\leq
 \frac{1}{\rateTpo}\ln\nmdl
+ \frac{1}{\rateTpo}(\ln(T) + 1.6) \\
&\leq \sqrt{\frac{2\nmdl(T+1)}{\ln\nmdl}}\left(1 + \frac{\ln\nmdl}{\ln (T+4)} + \underbrace{1.6/\ln (T+4)}_{\leq 1}\right)
\leq \sqrt{2(T+1)\nmdl\ln\nmdl}\left(\frac{2}{\ln \nmdl} + \frac{1}{\ln T}\right)
\,,
\\
\ln\frac{\fixtoT}{\mixttoT}
&\leq
\sqrt{2(T+1)\nmdl\ln\nmdl}\left(
\ln(T+3) + \nshift\left(\frac{2}{\ln \nmdl}+\frac{1}{\ln T}\right)
\right) \\
&\quad
+ \frac{5}{4}\frac{\ln\nmdl}{\nmdl}(1+\ln T)^3
+ \frac{N}{2}\ln(T+1)\,
\end{align*}
which was to be proven.
\end{proof}

\section{Failure of EG and OGD}\label[appendix]{app:thm:eg-lower}

\newcommand{\idxa}{a}
\newcommand{\idxb}{b}

\begin{proof}(Theorem~\ref{thm:eg-lower})
We start by proving the claim for EG.
The theorem will follow by considering two examples. For the first,
let $p^\idxa_t = 0$ and $p^\idxb_t = 1$ for all $t$. Then
according to the EG update rule in \cref{eq:EGupdate}:
\begin{align}
\wmdl^\idxa_t \propto \wmdl^\idxa_1  \qquad \text{ and } \qquad
\wmdl^\idxb_t \propto \wmdl^\idxb_1 \exp\left(\rate \sum_{s=1}^t \frac{1}{\wmdl^\idxb_s}\right)\,.
\label{eq:fail-propto}
\end{align}
It is easy to see that $w^\idxb_s \in [1/2, 1]$,
 which means that
\begin{align*}
w^\idxb_t = \frac{\exp\left(\rate \sum_{s=1}^t \frac{1}{w^\idxb_s}\right)}{1 + \exp\left(\rate \sum_{s=1}^t \frac{1}{w^\idxb_s}\right)}
\leq \frac{\exp(2\rate t)}{1 + \exp(2\rate t)}\,.
\end{align*}
The best mixture in hindsight in this case assigns all mass to the second expert and suffers no loss. Therefore
the regret
\begin{align*}
\cR_T = \sum_{t=1}^T \ln\left(\frac{1}{w^\idxb_t}\right)
\geq \sum_{t=1}^{\min\{1/\rate, T\}} \ln\left(\frac{1 + \exp(2)}{\exp(2)}\right)
\geq \frac{1}{10} \min\left\{\frac{1}{\rate}, T\right\} \,,
\end{align*}
which proves the result for small learning rates $\rate \leq T^{\varepsilon - 1}$. From now on suppose that $T$ is reasonably large and $\rate \geq T^{\varepsilon - 1}$.
Now we consider a different sequence of expert predictions. Let the first expert be a poor predictor for the first $T/2$ rounds, and subsequently
let the prediction quality of each expert alternate. Formally,
\begin{align*}
p^\idxa_t &= \begin{cases}
0 & \text{if } t \leq T/2 \\
1 & \text{if } t > T/2 \text{ and } t \text{ is even} \\
0 & \text{otherwise}\,.
\end{cases} &
p^\idxb_t &= \begin{cases}
1 & \text{if } t \leq T/2 \\
0 & \text{if } t > T/2 \text{ and } t \text{ is even} \\
1 & \text{otherwise}\,.
\end{cases}
\end{align*}
Notice that this is the same sequence as considered in the first part until $t = T/2$. Therefore by \cref{eq:fail-propto} for $t = T/2+1$ we have
\begin{align*}
w^\idxa_{T/2+1} = \left(1 + \exp\left(\rate \sum_{s=1}^{T/2} \frac{1}{w_s^\idxb}\right)\right)^{-1} \leq \exp(-\rate T/2)\,.
\end{align*}
This means that in a single round the predictor suffers loss $-\log \exp(-\rate T/2) = \rate T/2$, which may already be quite large. But there are still many rounds
to go and things do not get better. In round $T/2+2$ we have
\begin{align}
\label{eq:fail-2}
w^\idxb_{T/2+2}
&= \frac{w^\idxb_{T/2+1}}{w^\idxb_{T/2+1} + w^\idxa_{T/2+1} \exp\left(\rate / w^\idxa_{T/2+1}\right)}
\leq \frac{\exp\left(-\rate / w^\idxa_{T/2+1}\right)}{w^\idxa_{T/2+1}} \\
&\leq \frac{\exp\left(-\rate \exp(\rate T/2))\right)}{\exp(-\rate T/2)}
\leq \exp(-\rate T / 2)\, \nonumber
\end{align}
and so on. Therefore the loss of the EG algorithm over the final $T/2$ rounds is at least $\rate T^\idxb / 4 = \Omega(T^{1+\varepsilon})$.
For this sequence the loss of the best mixture in hindsight is $O(T)$ and hence the regret of EG is at least $\Omega(T^{1+\varepsilon})$, which completes the proof for EG.
Moving to gradient descent. We use a similar example where in rounds $t < T$ the first expert has $p^\idxa_t = 0$ and the second has $p^\idxb_t = 1$.
An easy calculation shows that
\begin{align*}
w^\idxb_{t+1}
= \max\left\{1,\, w^\idxb_t + \frac{\rate}{2 w^\idxb_t}\right\}
\geq \max\left\{1, \, w^\idxb_t + \frac{\rate}{2}\right\}\,.
\end{align*}
Therefore since $w^\idxb_1 = 1/2$, if $T > 1 + 1/\rate$, then $w^\idxb_T = 1$. In this case let $p^\idxa_T = 1$ and $p^\idxb_T = 0$, which leads to infinite regret.
On the other hand if $T \leq 1 + 1/\rate$, then let $p^\idxa_T = 0$ and $p^\idxb_T = 1$ so that the minimum loss in hindsight vanishes. Therefore the regret of OGD
is at least
\begin{align*}
\cR_T = \sum_{t=1}^T \ln\left(\frac{1}{w^\idxb_t}\right) \geq \sum_{t=1}^{T/2} \ln\left(\frac{1}{1/2 + \rate t}\right) = \Omega(T)\,,
\end{align*}
where we used the fact that $w^\idxb_t \geq 1/2$ so that $\max\{1, \, w^\idxb_t + \rate / (2 w_t^\idxb)\} \leq \max\{1,\, w^\idxb_t + \rate\}$.
\end{proof}

\begin{remark}
Notice that the second inequality in \cref{eq:fail-2} is rather loose when $\rate$ is large. In fact for learning rates $\rate = O(T^{\varepsilon - 1})$ one can show
the regret of EG grows \textit{super-exponentially}.
\end{remark}

\section{A multi-learning-rate Soft-Bayes}\label[appendix]{sec:multirate}

ML-Prod~\citep{gaillard2014mlprod} was designed for [0,1] losses for linear optimization.
It uses one learning rate per expert to be able to track the best of them
with a loss that does not depend on the other experts.
Interestingly, using the gradient trick on Prod with excess losses also leads
to the same update as Soft-Bayes
since the linearized loss of the algorithm is
$\sum_{i\leq \nmdl } \wmdlit \partial (-\ln \mixtt)/\partial \wmdlit = \sum_i \wmdlit (-\mdlit/\mixtt) = -1$.
We provide similar results as for ML-Prod, but with the online correction rule.
First, using fixed learning rates $\ratei$, we define the mixture as:
\begin{align*}
\mixttoT &:= \sum_{i=1}^\nmdl \wmdli_1 \prod_{t=1}^T ((1-\ratei)\mixtt + \ratei\mdlit) \\
\mixtt &= \frac{\mixttot}{\mixtpret} = \frac{\sum_i \wmdlit \ratei \mdlit}{\sum_i \wmdlit \ratei} \\
\wmdlitpo &:=\wmdli_1\frac{\prod_{t=1}^T ((1-\ratei)\mixtt + \ratei\mdlit)}{\mixttoT}= \wmdlit\left(1-\ratei+\ratei\frac{\mdlit}{\mixtt}\right)
\end{align*}
where the equation for $\mixtt$ follows from $\mixtt = \sum_i \wmdlit[(1-\ratei)\mixtt + \ratei\mdlit]$ and algebra.
Observe that for now the weights are kept normalized.
For time-dependent learning rates, we apply the online correction
rule~\cref{eq:sbonlineupdate} to the update of the weights, which gives the following:
\begin{align}
\mixtt &:= \frac{\mixttot}{\mixtpret} = \frac{\sum_i \wmdlit \rateit \mdlit}{\sum_i \wmdlit \rateit} \label{eq:mlsbonlinepred}\\
\wmdlitpo &:= \wmdlit\left(1-\rateit+\rateit\frac{\mdlit}{\mixtt}\right)\frac{\rateitpo}{\rateit} + \left(1-\frac{\rateitpo}{\rateit}\right)\wmdli_1\,.
\label{eq:mlsbonlineupdate}
\end{align}
All the properties of Lemma~\ref{lem:sbonlineupdate} still hold, just with additional indices on the learning rates.
Furthermore, the weights always remain positive since $\mdlit \geq 0$.

\begin{lemma}\label{lem:mlsb}
For the algorithm defined by \cref{eq:mlsbonlinepred,eq:mlsbonlineupdate},
with sequences $\rateit\in(0,1)$ of monotonically decreasing learning rates,
we have the following regret guarantee against the best fixed convex combination
of the experts $\fix$ in hindsight such that $\fixt := \sum_{i=1}^N\wfixi\mdlit$:
\begin{align*}
\ln\frac{\fixtoT}{\mixttoT} \leq
\sum_{i=1}^\nmdl \wfixi\left[ \frac{1}{\ratebariTpo}\ln\frac{\wmdliTpo}{\wmdli_1}
+ \sum_{t=1}^T\ratebarit\left(\frac{\mdlit}{\mixtt}-1\right)^2
+\ln\frac{\wmdliTpo}{\wmdli_1} \right]\,, \qquad
\text{with }\quad \ratebarit := \frac{\rateit}{1-\rateit}.
\end{align*}
\end{lemma}
\begin{proof}
Using $\ln(x)\leq x-1$, then
 Lemma~\ref{lem:lnratesimple} on the second line,
 the \proptelescoperate{}
 on the third line,
 and telescoping on the fourth line
  we have:
\begin{align*}
\sum_{t=1}^T \ln\frac{\fixt}{\mixtt} &\leq \sum_t \left(\frac{\fixt}{\mixtt}-1\right)
= \sum_t \sum_{i=1}^\nmdl \wfixi\left(\frac{\mdlit}{\mixtt}-1\right)
= \sum_i  \wfixi\sum_t\left(\frac{\mdlit}{\mixtt}-1\right)
 \\
&\leq \sum_i \wfixi \sum_t \left[  \frac{1}{\rateit}\ln\left(1-\rateit+\rateit\frac{\mdlit}{\mixtt}\right) + \frac{\rateit}{1-\rateit}\left(\frac{\mdlit}{\mixtt}-1\right)^2 \right]\\
&\leq \sum_i \wfixi\sum_t\left[
\left(\frac{1}{\rateitpo}\ln\frac{\wmdlitpo}{\wmdli_1}
-\frac{1}{\rateit}\ln\frac{\wmdlit}{\wmdli_1}
\right)
+ \ratebarit\left(\frac{\mdlit}{\mixtt}-1\right)^2 \right]\\
&= \sum_i \wfixi\left[ \frac{1}{\rateiTpo}\ln\frac{\wmdliTpo}{\wmdli_1}
+ \sum_{t=1}^T\ratebarit\left(\frac{\mdlit}{\mixtt}-1\right)^2 \right]\\
&= \sum_i \wfixi\left[ \frac{1-\rateiTpo}{\rateiTpo}\ln\frac{\wmdliTpo}{\wmdli_1}
+ \sum_{t=1}^T\ratebarit\left(\frac{\mdlit}{\mixtt}-1\right)^2
+\ln\frac{\wmdliTpo}{\wmdli_1} \right]\,, 
\end{align*}
which with $(1-\rateiTpo)/\rateiTpo = \ratebariTpo$ proves the claim.
\end{proof}
Unfortunately, applying the correction rule with a different ratio for each $i$ means that the weights may not be normalized anymore, so we cannot use $\wmdliTpo\leq 1$,
but it can be shown that the correction rule still ensures that they do not grow
faster than $O(\ln\ln T)$ if the ratios $\mdlit/\mixtt$ are bounded,
as for ML-Prod.

\ifarxiv

    More precisely, it can be shown that $\wmdliTpo$ is bounded above by $O(\sum_i \ln (1/\rateiTpo))$: From \cref{eq:mlsbonlineupdate}, using $\rateitpo/\rateit \leq 1$ together with $\rateit < 1$ (ensuring $1-\rateit \geq 0$), and $1-1/x \leq \ln x$ we can deduce
    \begin{align*}
    \sum_i \wmdlitpo &\leq \sum_i \left(\wmdlit\left(1-\rateit + \rateit \frac{\mdlit}{\mixtt}\right)\cancel{\frac{\rateitpo}{\rateit}} + \wmdli_1\ln \frac{\rateit}{\rateitpo} \right)\\
    &\leq
    \sum_i \wmdlit + \sum_i \wmdli_1\ln \frac{\rateit}{\rateitpo}\\
    &\leq \sum_i \wmdli_1 + \sum_i \wmdli_1\ln \frac{\ratei_1}{\rateitpo}\\
    &\leq \sum_i \wmdli_1 + \sum_i \wmdli_1\ln \frac{\ratebari_1}{\ratebaritpo}\,,\\
    \forall j \in [\nmdl]: 0 \leq \wmdl^j_{T+1} &\leq \sum_i \wmdli_1 \left(1+ \ln \frac{\ratebari_1}{\ratebariTpo}\right)\,,
    \end{align*}
    where we used
    $\sum_i \wmdlit\rateit\frac{\mdlit}{\mixtt}= \sum_i \wmdlit\rateit$ from
     \cref{eq:mlsbonlinepred} on the second line and telescoping over $t$ on the third line.

    Finally, let $V^i_k := \sum_{t\leq k} \left(\frac{\mdlit}{\mixtt}-1\right)^2$.
    Then, from \cref{lem:mlsb}, setting $\ratebarit := \sqrt{\frac{\ln(\nmdl)/2}{\ln \nmdl + V^i_{t-1}}}$ can be shown to lead to an upper bound on the regret of
    \begin{align*}
    O\left(\left[\ln \nmdl + \ln\left(1+\max_j\ln (1+ V^j_T)\right)\right]\sqrt{1+ \frac{V^i_T}{\ln\nmdl}}\right)
    \end{align*}
    simultaneously for all $i$.

\fi

\section{Sequence prediction with experts with disjoint supports}\label[appendix]{sec:disjoint}

\newcommand{\idxt}{\hat{i}_t} 
\newcommand{\idxs}{\hat{i}_s}
\newcommand{\idxtpo}{\hat{i}_{t+1}}
\newcommand{\mdlitt}{\mdl^{\idxt}_t}
\newcommand{\wmdlitt}{\wmdl^{\idxt}_t}
\newcommand{\wmdlittpo}{\wmdl^{\idxt}_{t+1}}
\newcommand{\countit}{n^i_t}

In this section we consider that the experts have disjoint supports,
that is, for any observation exactly one expert predicts it with positive probability:%
\footnote{But a single expert can still place positive probability
over several observations, as long as there is no overlap with any other expert.}
$\forall t\in[1..T], |\{i\in[1..\nmdl]| \mdlit > 0\}| = 1$.
This happens in particular if the experts
are designed so that each expert $i$ predicts the symbol of index $i$,
that is $\forall t, i:\mdlit(\obs_t=i) = 1$ when considering that $\allObs = [1..\nmdl]$.

In this setting, we show that the Soft-Bayes rule with a learning rate $\ratet$ of $\frac{1}{t+c}$ recovers exactly some well-known density estimators such as
Laplace's rule of succession and the minmax-optimal KT estimator.

Let $\idxt$ be the index of the model that places positive probability for the current observation $\obs_t$ at time $t$.
Then $\mixtt = \sum_i \wmdlit\mdlit = \wmdlitt\mdlitt$.

Let $\countit := \sum_{s=1}^{t} \indicator{\idxs = i}$ be the number of times up to $t$ where the expert $i$ is correct.
Then we have the following property.
\begin{theorem}
If the experts have disjoint supports and uniform prior $\wmdli_1 = 1/\nmdl$, then using a learning rate $\ratet := \frac{1}{t+c}$ makes the mixture predict
\begin{align*}
\forall t:
\mixttpo(\obs_{t+1}=i) = \frac{\countit + \frac{c}{\nmdl}}{t+c}\mdlitpo.
\end{align*}
\end{theorem}
\begin{proof}
We proceed by induction on the weights:
\begin{align*}
\wmdlittpo &= \wmdlitt(1-\ratet + \ratet\frac{\mdlitt}{\mixtt}) = \wmdlit(1-\ratet) + \ratet \\
\forall i\neq\idxt: \wmdlitpo &= \wmdlit(1-\ratet) \\
\text{that is }\forall i: \wmdlitpo &= \wmdlit(1-\ratet) + \ratet\indicator{\idxt=i}
\end{align*}
Now with $\ratet = \frac{1}{t+c}$, observe%
\footnote{Interestingly, this property exists only for this type of learning rate,
and not for example for $\ratet \propto \frac{1}{\sqrt{t}}$.}
 that $\frac{\ratet}{1-\ratet} = \frac{1}{t-1+c} = \ratetmo$. Then
\begin{align*}
\wmdlitpo &= \wmdli_1(1-\ratet) + \ratet\indicator{\idxt=i} \\
\frac{1}{\ratet} \wmdlitpo &= \frac{1-\ratet}{\ratet}\wmdli_1 + \indicator{\idxt=i} \\
&= \frac{1}{\ratetmo}\wmdlit + \indicator{\idxt=i} \\
&= \frac{1}{\ratez}\wmdli_1 + \sum_{s=1}^t \indicator{\idxs=i} \quad\text{(by induction)} \\
\frac{1}{t+c}\wmdlitpo&= \frac{c}{\nmdl} + \countit, \\
\text{hence}\quad \wmdlitpo &= \frac{\countit + \frac{c}{\nmdl}}{t+c}
\end{align*}
which with $\mixttpo = \wmdlitpo \mdlitpo$ for $i = \idxtpo$
proves the claim.
\end{proof}

In particular, for experts such that $\mdlit\in\{0,1\}$
and compared to the best constant convex combination of the experts in hindsight
(still with disjoint support),
\begin{itemize}
    \item setting $c=1$ recovers Perks' estimator~\citep{perks1947indifference,hutter2013sad},
    with a regret of $O(\nbestset\ln T)$ where $\nbestset$ is the number of experts
    that make at least one good prediction,
    \item setting $c=\nmdl$ recovers Laplace's rule of succession,
    with a regret of $O(\nmdl\ln \frac{T}{\nmdl})$,
    \item setting $c = \nmdl/2$ recovers the KT estimator~\citep{krichevsky1981performance}, with a regret of $O(\frac{\nmdl}{2}\ln T)$.
\end{itemize}
See \citep{hutter2013sad} for more details and comparison of these estimators.

We can draw an interesting parallel between these estimators and the different learning rates for competing against a fixed combination of the experts (with non-disjoint supports).

Indeed Perks' estimator with learning rate $\ratet = \frac{1}{t+1}$ is a sparse estimator: It pays a cost of $\log t$ each time a symbol is seen for the first time, just like a learning rate of $\ratet = \frac{1}{\sqrt{t}}$ pays a cost of $\sqrt{t}$ for convex combinations when an expert is good for the first time.
The KT estimator with a learning rate of $\frac{1}{t+\nmdl/2}$ minimizes the worst case where all symbols must be introduced, similarly to a learning rate of $\frac{1}{\sqrt{t\nmdl}}$ for convex combinations.

\vskip 0.2in
\ifarxiv
    \bibliographystyle{unsrtnat}
\fi
\bibliography{refs}

\end{document}